\documentclass[twoside,11pt]{article}

\usepackage{jmlr2e}

\newtheorem{assumption}[theorem]{Assumption}

\usepackage{dirtytalk}
\usepackage{color}
\usepackage{amsmath}
\usepackage{mathtools}
\usepackage{graphicx}
\usepackage{subcaption}
\usepackage{multirow}
\usepackage{booktabs}

\newcommand{\sun}{\xi}

\newcommand{\argmin}{\mathop{\mathrm{arg\,min\,\,}}}
\newcommand{\R}{\mathbb{R}}
\newcommand{\Rd}{\mathbb{R}^d}
\newcommand{\su}{\mathsf{u}}
\newcommand{\sv}{\mathsf{v}}
\newcommand{\se}{\mathsf{e}}
\newcommand{\Bigo}{\mathcal{O}}
\newcommand{\dtfrac}[1]{\frac{d #1}{dt}}

\definecolor{darkred}{rgb}{.7,0,0}

\usepackage{lastpage}
\jmlrheading{21}{2020}{1-\pageref{LastPage}}{6/19; Revised 9/20}{12/20}{19-466}{Nikola B. Kovachki and Andrew M. Stuart}

\ShortHeadings{Continuous Time Analysis of Momentum Methods}{Kovachki and Stuart}
\firstpageno{1}

\begin{document}

\title{Continuous Time Analysis of Momentum Methods}

\author{\name Nikola B. Kovachki \email nkovachki@caltech.edu \\
       \addr Computing and Mathematical Sciences\\
       California Institute of Technology\\
       Pasadena, CA 91125, USA
       \AND
       \name Andrew M. Stuart \email astuart@caltech.edu \\
       \addr Computing and Mathematical Sciences\\
       California Institute of Technology\\
       Pasadena, CA 91125, USA}

\editor{Suvrit Sra}

\maketitle

\begin{abstract}
Gradient descent-based optimization methods underpin the parameter training
of neural networks, and hence comprise a significant component in the
impressive test results found in a number of applications.
Introducing stochasticity is key to their success in practical problems, 
and there is some understanding
of the role of stochastic gradient descent in this context. Momentum modifications of gradient descent 
such as Polyak's Heavy Ball method (HB) and Nesterov's method of accelerated 
gradients (NAG), are also widely adopted. In this work our focus is on
understanding the role of momentum in the training of neural networks, concentrating
on the
common situation in which the momentum contribution is fixed at each step of the
algorithm. To expose 
the ideas simply we work in the deterministic setting. 

Our approach is to derive continuous time approximations of the
discrete algorithms; these continuous time approximations provide
insights into the mechanisms at play within the discrete algorithms.
We prove three such approximations. Firstly we show that standard 
implementations of fixed momentum methods approximate a time-rescaled
gradient descent flow, asymptotically as the learning rate shrinks to
zero; this result does not distinguish momentum methods from pure
gradient descent, in the limit of vanishing learning rate.
We then proceed to prove two results aimed at understanding 
the observed practical advantages of fixed momentum methods over 
gradient descent,
when implemented in the non-asymptotic regime with fixed small, but
non-zero, learning rate.  We achieve this by proving approximations to
continuous time limits in which the small but fixed learning rate
appears as a parameter; this is known as the method of {\em modified
equations} in the numerical analysis literature, recently rediscovered
as the {\em high resolution ODE} approximation in the machine
learning context. In our second result we show that the
momentum method is approximated by a continuous time gradient flow, 
with an additional momentum-dependent second order time-derivative
correction, proportional to
the learning rate; this may be used to explain the stabilizing effect of
momentum algorithms in their transient phase.
Furthermore in a third result we show that the momentum methods 
admit an exponentially attractive 
invariant manifold  on which the dynamics reduces, approximately,
to a gradient flow with respect to a modified loss function, equal 
to the original loss function plus a small perturbation proportional to the
learning rate; this small correction provides convexification of
the loss function and encodes additional robustness present in
momentum methods, beyond the transient phase.
\end{abstract}


\section{Introduction}
\label{sec:I}

\subsection{Background and Literature Review}
\label{ssec:B}

At the core of many machine learning tasks is solution of the optimization problem
\begin{equation}
\label{eq:min}
\argmin_{u \in \Rd} \Phi(u)
\end{equation}
where \(\Phi: \Rd \to \R\) is an objective (or loss) function that is, in general, non-convex 
and differentiable. Finding global minima of such objective functions is an important
and challenging task with a long history, one in which the use of stochasticity
has played a prominent role for many decades, with papers in the early
development of machine learning \cite{geman1987stochastic,styblinski1990experiments},
together with concomitant theoretical analyses
for both discrete \cite{bertsimas1993simulated}
and continuous problems \cite{kushner1987asymptotic,kushner2012stochastic}.  
Recent successes in the training of deep neural networks have built on this
older work, leveraging the enormous computer power now available, together with
empirical experience about good design choices for the architecture of the networks;
reviews may be found in \cite{deeplearningbook,deeplearningnature}. 
Gradient descent plays a prominent conceptual role in many algorithms, following
from the observation that the equation
\begin{equation}
\label{eq:GD}
\frac{du}{dt}=-\nabla \Phi(u)
\end{equation}
will decrease $\Phi$ along trajectories.
The most widely adopted methods use stochastic gradient decent (SGD), a concept
introduced in \cite{robbinsmonroe}; the basic idea is to use
gradient decent steps based on a noisy approximation to the gradient of \(\Phi\). 
Building on deep work in the convex optimization literature, momentum-based
modifications to stochastic gradient decent have also become widely used in
optimization. Most notable
amongst these momentum-based methods are the Heavy Ball Method (HB), 
due to \cite{polyakhevyball}, and Nesterov's method of accelerated gradients (NAG) \cite{originalnesterov}.
To the best of our knowledge, the first application of HB to neural network training appears in \cite{firsthb}. More 
recent work, such  as \cite{importanceofinitmom}, has even argued for the indispensability of such momentum
based methods for the field of deep learning. 

From these two basic variants on gradient decent, there have come a plethora of 
adaptive methods, incorporating momentum-like ideas, such 
as Adam \cite{adam}, Adagrad \cite{adagrad}, and RMSProp \cite{rmsprop}. There is no
consensus on which method performs best and results vary based on application. The recent work of
\cite{marginaladaptive} argues that the rudimentary, non-adaptive schemes SGD, HB, and NAG result 
in solutions with the greatest generalization performance for supervised learning applications with
deep neural network models.

There is a natural physical analogy for momentum methods, namely that they relate to
a damped second order Hamiltonian dynamic with potential $\Phi$:
\begin{equation}
\label{eq:HAM}
m\frac{d^2u}{dt^2}+\gamma(t)\frac{du}{dt}+\nabla \Phi(u)=0.
\end{equation}
This perspective goes back to Polyak's original work \cite{polyakhevyball,polyakoptimization}
and was further expanded on in \cite{wronghb}, although no proof was given. 
For NAG, the work of \cite{odenesterov} proves that the method approximates a damped Hamiltonian system 
of precisely this form, with a time-dependent damping coefficient.
The analysis in \cite{odenesterov} holds  when 
the momentum factor is chosen according to the rule
\begin{equation}
\label{eq:lambda_time}
\lambda = \lambda_n =  \frac{n}{n + 3},
\end{equation}
where $n$ is the iteration count; this choice was proposed
in the original work of \cite{originalnesterov} and results in
a choice of $\lambda$  
which is asymptotic to $1$.  In the setting where
\(\Phi\) is \(\mu\)-strongly convex, it is proposed in
\cite{nesterovbook} that the momentum factor
is fixed and chosen close to $1$; specifically it is proposed that
\begin{equation}
\label{eq:lambda_mu}
\lambda = \frac{1 - \sqrt{\mu h}}{1+ \sqrt{\mu h}}
\end{equation}
where \(h > 0\) is the time-step (learning rate). 
In \cite{wilsoneq}, a limiting equation for both HB and NAG of the
form 
\[\ddot{u} + 2\sqrt{\mu}\dot{u} + \nabla \Phi(u) = 0\]
is derived under the assumption that \(\lambda\) is fixed with respect
to iteration number $n$, and dependent on the time-step $h$ as specified
in \eqref{eq:lambda_mu};
convergence is obtained to order \(\mathcal{O}(h^{1/2})\).
Using insight from this limiting equation it is possible to
choose the optimal value of $\mu$ to maximize the convergence
rate in the neighborhood of a locally strongly convex objective
function.  Further related work is developed in \cite{shi2018understanding} 
where separate limiting equations for HB and NAG  are derived both in 
the cases of \(\lambda\)
given by \eqref{eq:lambda_time} and \eqref{eq:lambda_mu}, obtaining 
convergence to order \(\mathcal{O}(h^{3/2})\). Much work has 
also gone into analyzing these methods in the discrete setting, without
appeal to the continuous time limits, see 
\cite{hu2017dissipativity,lessard2016analysis}, as well as in the 
stochastic setting, establishing
how the effect on the generalization error, for 
example,  \cite{gadat2018stochastic,loizou2017linearly,yang2016unified}.
In this paper, however, our focus is on the use of continuous time limits
as a methodology to explain optimization algorithms.

In many machine learning applications, especially for deep learning,
NAG and HB are often used with a constant momentum factor \(\lambda\) 
that is chosen independently of the iteration count $n$
(contrary to \eqref{eq:lambda_time}) and independently of the 
learning rate $h$ (contrary to \eqref{eq:lambda_mu}). 
In fact, popular books on the subject such as \cite{deeplearningbook}
introduce the methods in this way, and popular articles, such as \cite{originalresnet} to name one of many, 
simply state the value of the constant momentum factor used in their experiments. 
Widely used deep learning libraries such as Tensorflow \cite{tensorflow} and 
PyTorch \cite{pytorch} implement the methods with a fixed choice 
of momentum factor.
Momentum based methods used in this way,
with fixed momentum, have not been carefully analyzed. 
We will undertake
such an analysis, using ideas from numerical analysis, and in particular
the concept of {\em modified equations} \cite{griffiths1986scope,chartier2007numerical} 
and from the theory of {\em attractive invariant 
manifolds} \cite{hirsch2006invariant,wiggins2013normally};   
both ideas are explained in the text \cite{stuart1998dynamical}. 
It is noteworthy that the {\em  high resolution ODE approximation}
described in \cite{shi2018understanding} may be viewed as
a rediscovery of the method of modified equations.
We emphasize the fact that our work is not at odds 
with any previous analyses of these methods, rather, we consider a
setting which is widely adopted in deep learning applications
and has not been subjected to continuous time analysis to date.

\begin{remark}
Since publication of this article in \cite{kovachki2021continuous},
we became aware of related, and earlier, work by \cite{farazmand2018multiscale}.
Farazmand starts from the Bregman Lagrangian introduced in 
\cite{wibisono2016variational} and uses ideas from geometric singular 
perturbation theory to derive an invariant manifold. The work leads
to a more general description of the invariant manifold than the one 
given by our equation \eqref{eq:approxg2}. Farazmand's work was published
in \cite{farazmand2020multiscale}.
\end{remark}

\subsection{Our Contribution}

We study momentum-based optimization algorithms for the minimization task
\eqref{eq:min}, with learning rate independent momentum, fixed at every
iteration step, focusing on deterministic methods for clarity of exposition.
Our approach is to derive continuous time approximations of the discrete 
algorithms; these continuous time approximations provide insights into the 
mechanisms at play within the discrete algorithms. 
We prove three such approximations. The first shows that the asymptotic
limit of the momentum methods, as learning rate approaches zero,
is simply a rescaled gradient flow \eqref{eq:GD}. The second two
approximations include small perturbations to the rescaled gradient flow,
on the order of the learning rate, and give insight into the behavior
of momentum methods when implemented with momentum and fixed learning
rate. Through these approximation theorems, and accompanying
numerical experiments, we make the following contributions to the
understanding of momentum methods as often implemented within
machine learning:

\begin{itemize}

\item We show that momentum-based methods
with a fixed momentum factor, satisfy, in the continuous-time limit obtained
by sending the learning rate to zero, 
a rescaled version of the gradient flow equation \eqref{eq:GD}. 

\item We show that such methods also approximate a damped
Hamiltonian system of the form \eqref{eq:HAM}, with small mass $m$ (on
the order of the learning rate) and constant damping $\gamma(t)=\gamma$; 
this approximation has the same order of accuracy as the approximation 
of the rescaled equation \eqref{eq:GD} but provides a better 
qualitative understanding of the fixed learning rate momentum
algorithm in its transient phase.

\item We also show that, for the approximate Hamiltonian system, the dynamics 
admit an exponentially attractive invariant manifold, 
locally representable as a graph mapping co-ordinates to their velocities.
The map generating this graph describes 
a gradient flow in a potential which is a small (on the order of the learning rate)
perturbation of $\Phi$ -- see \eqref{eq:mp}; the correction to the potential is convexifying,
does not change the global minimum, and provides insight into the
fixed learning rate momentum algorithm beyond its initial transient
phase.

\item We provide numerical experiments which illustrate the foregoing considerations, for simple linear test problems, 
and for the MNIST digit classification problem; in the latter case we
consider SGD and thereby demonstrate that the conclusions of our theory
have relevance for understanding the stochastic setting as well.

\end{itemize}

Taken together our results are interesting because they demonstrate that
the popular belief that (fixed) momentum methods resemble the dynamics induced
by \eqref{eq:HAM} is misleading. Whilst it is true, the mass in the approximating
equation is small and
as a consequence understanding the dynamics as gradient flows \eqref{eq:GD}, 
with modified potential, is more instructive. 
In fact, in the first application of HB to neural networks 
described in \cite{firsthb}, the authors state that
\say{[their] experience has been that [one] get[s] the same solutions by setting [the momentum factor to zero]
and reducing the size of [the learning rate].} However our theorems should 
not be understood to imply that there is no practical difference between momentum 
methods (with fixed learning rate) and SGD. There is indeed a practical difference 
as has been demonstrated in numerous papers throughout the machine learning 
literature, and our experiments in Section \ref{ssec:DL} further confirm this. 
We show that while these methods have the same transient dynamics, they are approximated differently. Our results demonstrate that, although momentum methods
behave like a gradient descent algorithm, asymptotically, this algorithm
has a modified potential. Furthermore, although this modified potential  \eqref{eq:approxg2}
is on the order of the learning rate, the fact that the learning rate is often
chosen as large as possible, constrained by numerical stability, means that
the correction to the potential may be significant. Our results may be 
interpreted as indicating that the practical success of momentum methods stems 
from the fact that they provide a more stable 
discretization to \eqref{eq:GD} than the forward Euler method employed in SGD. The damped Hamiltonian 
dynamic \eqref{eq:visco}, as well the modified potential, give insight into
how this manifests. Our work gives further theoretical justification for the exploration of 
the use of different numerical integrators for the purposes of optimization such as those performed 
in \cite{scieur2017integration,betancourt2018symplectic,runge-kuttadisc}.

While our analysis is confined to the non-stochastic case to simplify the exposition, 
the results will, with some care, extend to the stochastic setting using ideas from
averaging and homogenization \cite{PavS} as well as continuum analyses of SGD as in \cite{weinansgd,sgdsemigroup};
indeed, in the stochastic setting, sharp uniform in
time error estimates are to be expected for empirical
averages \cite{mat10,bach17}.
To demonstrate that our analysis is 
indeed relevant in the stochastic setting, we train a deep autoencoder with mini-batching (stochastic)
and verify that our convergence results still hold. The details of this experiment 
are given in section \ref{ssec:DL}.
Furthermore we also confine our analysis to fixed learning rate, and impose global
bounds on the relevant derivatives of $\Phi$; this further simplifies the exposition of
the key ideas, but is not essential to them; with considerably more analysis the ideas
exposed in this paper will transfer to adaptive time-stepping methods and much less
restrictive classes of \(\Phi\).

The paper is organized as follows. Section \ref{sec:M}
introduces the optimization procedures and states the 
convergence result to a rescaled gradient flow. In section \ref{sec:ME} we
derive the modified, second-order equation and 
state convergence of the schemes to this equation. 
Section \ref{sec:IM}  asserts the existence of an attractive invariant manifold, 
demonstrating that it results in a gradient flow with respect to a small
perturbation of $\Phi$. In section \ref{ssec:DL}, we train a deep autoencoder,
showing that our results hold in a stochastic setting with Assumption \ref{assump:one} violated.
We conclude in section \ref{sec:C}.
All proofs of theorems are given in the appendices so that the ideas of the
theorems can be presented clearly within the main body of the text.

\subsection{Notation}
\label{ssec:N}

We use $|\cdot|$ to denote the Euclidean norm on $\R^d.$
We define \(f : \Rd \to \Rd\) by \(f(u) \coloneqq - \nabla \Phi(u)\) for any \(u \in \Rd\).
Given parameter $\lambda \in [0,1)$ we define
 \(\bar{\lambda} \coloneqq (1-\lambda)^{-1}\).

For two Banach spaces \(A,B\), and $A_0$ a subset in $A$, we denote by \(C^k(A_0;B)\) the set of \(k\)-times continuously differentiable functions
with domain \(A_0\) and range \(B\). For a function \(u \in C^k(A_0;B)\), we let 
\(D^ju\) denote its \(j\)-th (total) Fr{\'e}chet derivative for \(j=1,\dots,k\). 
For a function \({u \in C^k([0,\infty), \R^d)}\), we denote its 
derivatives by \(\frac{du}{dt}, \frac{d^2 u}{dt^2},\) etc. or equivalently by \(\dot{u}, \ddot{u},\) etc.

To simplify our proofs, we make the following assumption 
about the objective function.

\begin{assumption}
\label{assump:one}
Suppose \(\Phi \in C^3(\Rd;\R)\) with uniformly bounded derivatives. Namely,  
there exist constants \(B_0,B_1,B_2 > 0\) such that
\[\|D^{j-1}f\| = \|D^j \Phi \| \leq B_{j-1}\]
for \(j=1,2,3\) where \(\|\cdot\|\) denotes any appropriate operator norm.
\end{assumption}

We again stress that this assumption is not key to developing the ideas in this work, but is 
rather a simplification used to make our results global. Without Assumption \ref{assump:one},
and no further assumption on \(\Phi\) such as convexity, one could only hope to give local results 
i.e. in the neighborhood of a critical point of \(\Phi\). Such analysis could indeed 
be carried out (see for example \cite{carr2012applications}), but we choose not to do so here for 
the sake of clarity of exposition. In section \ref{ssec:DL}, we give a practical example 
where this assumption is violated and yet the behavior is as predicted
by our theory.

Finally we observe that the nomenclature ``learning rate'' is now prevalent
in machine learning, and so we use it in this paper; it refers to the 
object commonly referred to as ``time-step'' in the field of numerical analysis.

\section{Momentum Methods and Convergence to Gradient Flow}
\label{sec:M}

In subsection \ref{ssec:MR1} we state Theorem \ref{thm:conv_to_gf}
concerning the convergence of a class of momentum methods to a rescaled gradient
flow. Subsection \ref{ssec:Link} demonstrates that the HB and NAG methods
are special cases of our general class of momentum methods, and gives
intuition for proof of Theorem \ref{thm:conv_to_gf}; the proof itself
is given in Appendix A.  Subsection \ref{ssec:N1} contains a 
numerical illustration of  Theorem \ref{thm:conv_to_gf}. 

\subsection{Main Result}
\label{ssec:MR1}

The standard Euler discretization of \eqref{eq:GD} gives the discrete time optimization
scheme
\begin{equation}
\label{eq:gd}
\mathsf{u}_{n+1} = \mathsf{u}_n + h f(\mathsf{u}_n), \quad n=0,1,2,\dots\,.
\end{equation}
Implementation of this scheme requires an initial guess \(\mathsf{u}_0 \in \Rd\). 
For simplicity we consider a fixed learning rate \(h > 0\). 
Equation \eqref{eq:GD} has a unique solution
\(u \in C^3([0,\infty);\Rd)\) under Assumption \ref{assump:one} and
for $u_n=u(nh)$
$$\sup_{0 \le nh \le T} |\mathsf{u}_n-u_n| \le C(T)h;$$
see \cite{stuart1998dynamical}, for example. 

In this section we consider a general class of momentum methods for the minimization
task \eqref{eq:min} which can be written in the form, for some \(a \geq 0\) and
$\lambda \in (0,1)$,
\begin{align}
\label{eq:general_discrete}
\begin{split}
\mathsf{u}_{n+1} &= \mathsf{u}_n + \lambda(\mathsf{u}_n - \mathsf{u}_{n-1}) + hf(\mathsf{u}_n + a(\mathsf{u}_n - \mathsf{u}_{n-1})), \quad n=0,1,2,\dots\,, \\
\mathsf{u}_1 &= \mathsf{u}_0 + hf(\mathsf{u}_0)\,.
\end{split}
\end{align}
Again, implementation of this scheme requires an an initial guess 
\(\mathsf{u}_0 \in \Rd\).
The parameter choice \(a=0\) gives HB and \(a = \lambda\) gives NAG. 
In Appendix A we prove the following: 

\begin{theorem}
\label{thm:conv_to_gf}
Suppose Assumption \ref{assump:one} holds and let \(u \in C^3([0,\infty);\Rd)\) 
be the solution to 
\begin{align}
\label{eq:rgf}
\begin{split}
&\dtfrac{u} = -(1-\lambda)^{-1}\nabla \Phi(u) \\
&u(0) = \mathsf{u}_0
\end{split}
\end{align}
with \(\lambda \in (0,1)\). For \(n=0,1,2,\dots\)
let \(\su_n\) be the sequence given by \eqref{eq:general_discrete} 
and define \(u_n \coloneqq u(nh)\). Then for
any \(T \geq 0\), there is a constant \(C = C(T) > 0\) such that
\[\sup_{0 \leq nh \leq T} |u_n - \su_n| \leq Ch.\]
\end{theorem}

Note that \eqref{eq:rgf} is simply a
sped-up version of \eqref{eq:GD}: if \(v\) solves \eqref{eq:GD} and \(w\) solves 
\eqref{eq:rgf} then \(v(t) = w((1-\lambda)t)\) for any \(t \in [0,\infty)\).
This demonstrates that introduction of momentum in the form used within
both HB and NAG results in numerical methods that do not differ substantially
from gradient descent.

\subsection{Link to HB and NAG}
\label{ssec:Link}

The HB method is usually written as a 
two-step scheme taking the form (\cite{importanceofinitmom})
\begin{align*}
\mathsf{v}_{n+1} &= \lambda \mathsf{v}_n + h f(\mathsf{u}_n) \\
\mathsf{u}_{n+1} &= \mathsf{u}_n + \mathsf{v}_{n+1}
\end{align*}
with \(\mathsf{v}_0 = 0\), \(\lambda \in (0,1)\) the momentum factor, and \(h > 0\) the learning rate. We can re-write this update 
as 
\begin{align*}
\mathsf{u}_{n+1} &= \mathsf{u}_n + \lambda \mathsf{v}_n + hf(\mathsf{u}_n) \\
&= \mathsf{u}_n + \lambda (\mathsf{u}_n - \mathsf{u}_{n-1}) + h f(\mathsf{u}_n)
\end{align*}
hence the method reads 
\begin{align}
\label{eq:hb_discrete}
\begin{split}
\mathsf{u}_{n+1} &= \mathsf{u}_n + \lambda (\mathsf{u}_n - \mathsf{u}_{n-1}) + h f(\mathsf{u}_n) \\
\mathsf{u}_1 &= \mathsf{u}_0 + hf(\mathsf{u}_0). 
\end{split}
\end{align}

Similarly NAG is usually written as (\cite{importanceofinitmom})
\begin{align*}
\mathsf{v}_{n+1} &= \lambda \mathsf{v}_n + h f(\mathsf{u}_n + \lambda \mathsf{v}_n) \\
\mathsf{u}_{n+1} &= \mathsf{u}_n + \mathsf{v}_{n+1}
\end{align*}
with \(\mathsf{v}_0 = 0\). Define \(\mathsf{w}_n \coloneqq \mathsf{u}_n + \lambda \mathsf{v}_n\) then 
\begin{align*}
\mathsf{w}_{n+1} &= \mathsf{u}_{n+1} + \lambda \mathsf{v}_{n+1} \\
&= \mathsf{u}_{n+1} + \lambda (\mathsf{u}_{n+1} - \mathsf{u}_n)
\end{align*}
and
\begin{align*}
\mathsf{u}_{n+1} &= \mathsf{u}_n + \lambda \mathsf{v}_n + hf(\mathsf{u}_n + \lambda \mathsf{v}_n) \\
&= \mathsf{u}_n + (\mathsf{w}_n - \mathsf{u}_n) + h f(\mathsf{w}_n) \\
&= \mathsf{w}_n + h f(\mathsf{w}_n).
\end{align*}
Hence the method may be written as  
\begin{align}
\label{eq:nag_discrete}
\begin{split}
\mathsf{u}_{n+1} &= \mathsf{u}_n + \lambda(\mathsf{u}_n - \mathsf{u}_{n-1}) + hf(\mathsf{u}_n + \lambda(\mathsf{u}_n - \mathsf{u}_{n-1})) \\
\mathsf{u}_1 &= \mathsf{u}_0 + hf(\mathsf{u}_0). 
\end{split}
\end{align}

It is clear that \eqref{eq:hb_discrete} and \eqref{eq:nag_discrete} are special cases 
of \eqref{eq:general_discrete}
with \(a=0\) giving HB and \(a = \lambda\) giving NAG. 
To intuitively understand Theorem \ref{thm:conv_to_gf}, re-write \eqref{eq:rgf} as 
\[\frac{du}{dt} - \lambda \frac{du}{dt} = f(u).\]
If we discretize the \(du/dt\) term using forward differences and the \(-\lambda du/dt\) term using backward differences, 
we obtain
\[\frac{u(t+h) - u(t)}{h} - \lambda \frac{u(t) - u(t-h)}{h} \approx f(u(t)) \approx f \left( u(t) + ha \frac{u(t) - u(t-h)}{h} \right )\]
with the second approximate equality coming from the Taylor expansion of \(f\). This can be rearranged as
\[u(t+h) \approx u(t) + \lambda (u(t) - u(t-h)) + h f(u(t) + a(u(t)-u(t-h)))\]
which has the form of \eqref{eq:general_discrete} with the identification \(\mathsf{u}_n \approx u(nh)\).

\begin{figure}[t]
    \centering
    \begin{subfigure}[b]{0.31\textwidth}
        \includegraphics[width=\textwidth]{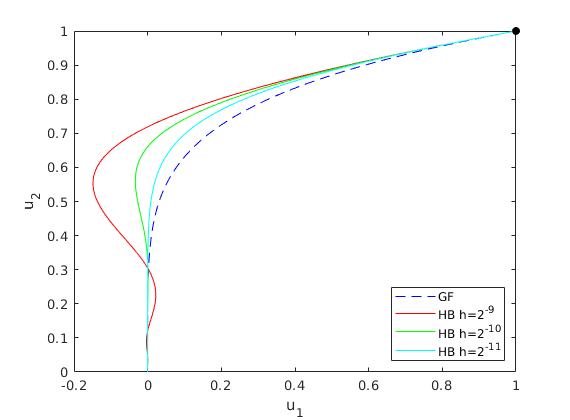}
        \caption{HB: \(\kappa = 5\)}
    \end{subfigure}
    ~ 
    \begin{subfigure}[b]{0.31\textwidth}
        \includegraphics[width=\textwidth]{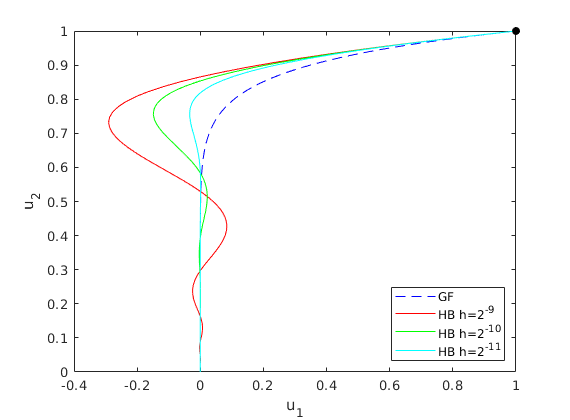}
        \caption{HB: \(\kappa = 10\)}
    \end{subfigure}
     ~ 
    \begin{subfigure}[b]{0.31\textwidth}
        \includegraphics[width=\textwidth]{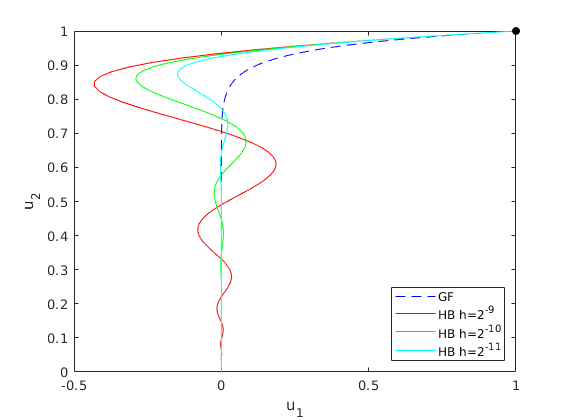}
        \caption{HB: \(\kappa = 20\)}
    \end{subfigure}

    \begin{subfigure}[b]{0.31\textwidth}
        \includegraphics[width=\textwidth]{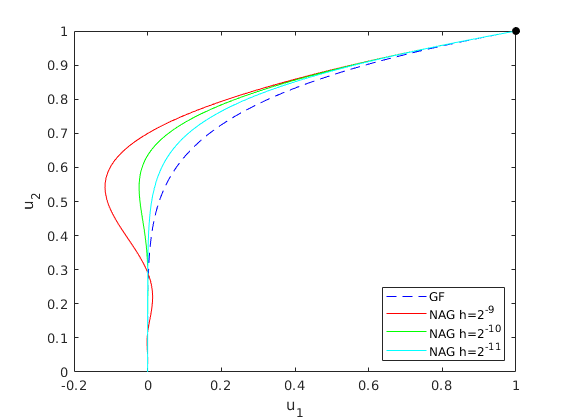}
        \caption{NAG: \(\kappa = 5\)}
    \end{subfigure}
    ~ 
    \begin{subfigure}[b]{0.31\textwidth}
        \includegraphics[width=\textwidth]{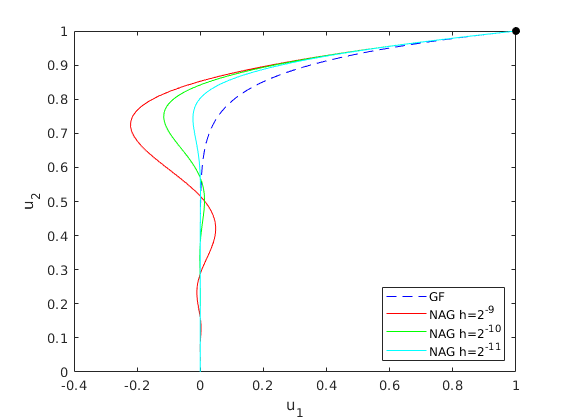}
        \caption{NAG: \(\kappa = 10\)}
    \end{subfigure}
     ~ 
    \begin{subfigure}[b]{0.31\textwidth}
        \includegraphics[width=\textwidth]{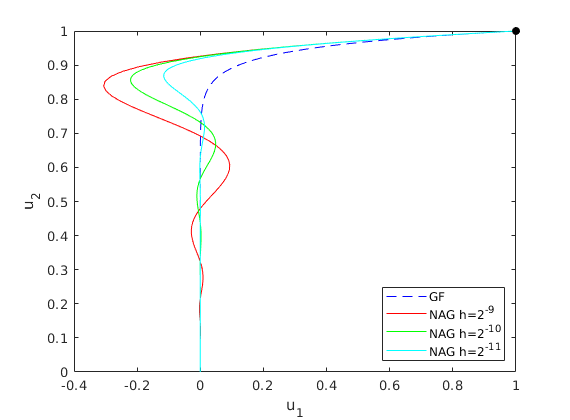}
        \caption{NAG: \(\kappa = 20\)}
    \end{subfigure}

    \caption{Comparison of trajectories for HB and NAG with the gradient flow \eqref{eq:rgf} on the two-dimensional problem
    \(\Phi(u) = \frac{1}{2} \langle u, Q u \rangle\) with \(\lambda=0.9\) fixed. We vary the condition 
    number of \(Q\) as well as the learning rate \(h\).} 
    \label{fig:gf_traj}
\end{figure}

\begin{figure}[t]
    \centering
    \begin{subfigure}[b]{0.31\textwidth}
        \includegraphics[width=\textwidth]{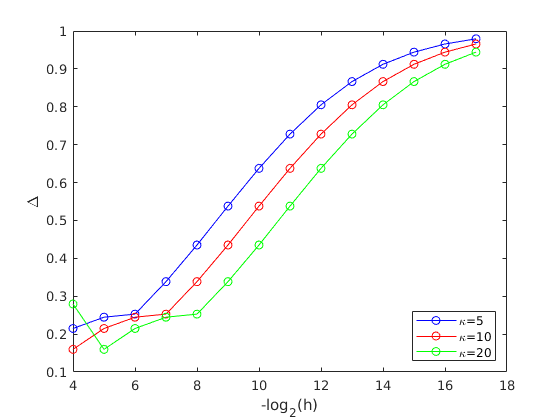}
        \caption{HB}
    \end{subfigure}
    ~ 
    \begin{subfigure}[b]{0.31\textwidth}
        \includegraphics[width=\textwidth]{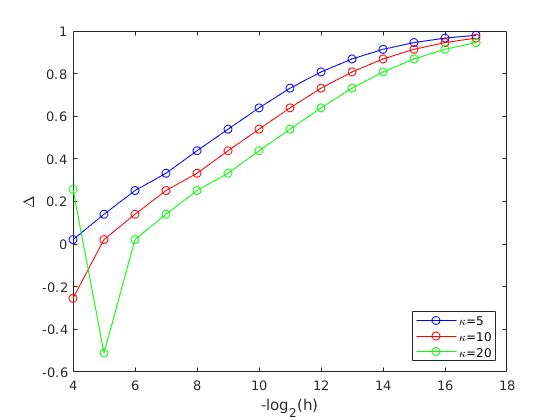}
        \caption{NAG}
    \end{subfigure}

    \caption{The numerical rate of convergence, as a function of the learning rate \(h\), of HB and NAG to the gradient flow \eqref{eq:rgf}
    for the problem described in Figure \ref{fig:gf_traj}.}
    \label{fig:gf_conv}
\end{figure}

\subsection{Numerical Illustration}
\label{ssec:N1}

Figure \ref{fig:gf_traj} compares trajectories of the momentum numerical method
\eqref{eq:general_discrete} with the rescaled gradient flow \eqref{eq:rgf}, for
the two-dimensional problem \(\Phi(u) = \frac{1}{2} \langle u, Q u \rangle\).
We pick \(Q\) to be positive-definite so that the minimum is achieved at the point \((0,0)^T\)
and make it diagonal so that we can easily control its condition number.
In particular, the condition number of \(Q\) is given as
\[\kappa = \frac{\max \{Q_{11}, Q_{22}\} }{\min \{Q_{11}, Q_{22}\} }.\]
We see that, as the condition number is increased, both HB and NAG exhibit 
more pronounced transient oscillations and are thus further away from 
the trajectory of \eqref{eq:rgf}, however, as the learning rate \(h\) is 
decreased, the oscillations dampen and the trajectories match more and more 
closely. This observation from Figure \ref{fig:gf_traj} is quantified in 
Figure \ref{fig:gf_conv} where we estimate the rate of convergence, as a function
of \(h\), which is defined as 
\[\Delta = \log_2 \frac{\|\su^{(h)} - u\|_\infty}{\|\su^{(h/2)} - u\|_\infty}\] 
where \(\su^{(\alpha)}\) is the numerical solution using time-step \(\alpha\). The figure shows
that the rate of convergence is indeed close to $1$, as predicted by our theory.
In summary the behavior of the momentum methods is precisely that of a rescaled
gradient flow, but with initial transient oscillations which capture momentum
effects, but disappear as the learning rate is decreased.
We model these oscillations in the next section
via use of a modified equation.

\section{Modified Equations}
\label{sec:ME}

The previous section demonstrates how the momentum methods approximate a
time rescaled version of the gradient flow \eqref{eq:GD}. In this section
we show how the same methods may also be viewed as approximations of
the damped Hamiltonian system \eqref{eq:HAM}, with mass $m$ on the
order of the learning rate, using the method of modified equations.
In subsection \ref{ssec:MR2} we state and discuss the main result of the section,
Theorem \ref{thm:conv_to_visco}. 
Subsection \ref{ssec:Link2} gives
intuition for proof of Theorem \ref{thm:conv_to_visco}; the proof itself
is given in Appendix B. And the section also contains
comments on generalizing the idea of modified equations. 
In subsection \ref{ssec:N2} we describe a numerical
illustration of Theorem \ref{thm:conv_to_visco}. 

\subsection{Main Result}
\label{ssec:MR2}

The main result of this section quantifies the sense in which momentum methods
do, in fact, approximate a damped Hamiltonian system; it is proved in Appendix B.

\begin{theorem}
\label{thm:conv_to_visco}
Fix $\lambda \in (0,1)$ and assume that $a \ge 0$ is chosen so that
$\alpha \coloneqq \frac12(1 + \lambda - 2a(1-\lambda))$ is strictly positive.
Suppose Assumption \ref{assump:one} holds and let \(u \in C^4([0,\infty);\Rd)\) be the solution to
\begin{align}
\label{eq:visco}
\begin{split}
& h \alpha \frac{d^2u}{dt^2} + (1-\lambda)\frac{du}{dt} = -\nabla \Phi(u) \\
&u(0) = \su_0, \quad 
\frac{du}{dt}(0) = \su_0'.
\end{split}
\end{align}
Suppose further that \(h \leq (1-\lambda)^2/2 \alpha B_1\).
For \(n=0,1,2,\dots\) let \(\su_n\) be the sequence given by \eqref{eq:general_discrete} and define \(u_n \coloneqq u(nh)\). Then for
any \(T \geq 0\), there is a constant \(C = C(T) > 0\) such that
\[\sup_{0 \leq nh \leq T} |u_n - \su_n| \leq Ch.\]
\end{theorem}

Theorem \ref{thm:conv_to_gf} demonstrates the same order of convergence, namely
${\mathcal O}(h)$, to the rescaled gradient flow equation \eqref{eq:rgf},
obtained from \eqref{eq:visco} simply by setting $h=0.$ In the standard method
of modified equations the limit system (here \eqref{eq:rgf})
is perturbed by small terms (in terms of the assumed small learning rate) 
and an increased  rate of convergence is obtained to the modified equation 
(here  \eqref{eq:visco}). In our setting however, because the small modification is to
a higher derivative (here second) than appears in the limit equation (here first
order), an increased 
rate of convergence is not obtained.  This is due to the nature of the modified 
equation, whose solution has derivatives that are inversely proportional to 
powers of \(h\); this fact is quantified in Lemma \ref{lemma:visco_bound} 
from Appendix B. It is precisely because the modified equation does not lead to
a higher rate of convergence that the initial parameter $\su_0'$ is arbitrary; the same
rate of convergence is obtained no matter what value it takes.

It is natural to ask, therefore, what
is learned from the convergence result in Theorem \ref{thm:conv_to_visco}. The
answer is that, although the modified equation \eqref{eq:visco} is approximated
at the same order as the limit equation \eqref{eq:rgf}, it actually contains
considerably more qualitative information about the dynamics of the system,
particularly in the early transient phase of the algorithm; this will be
illustrated in subsection \ref{ssec:N2}.
Indeed we will make a specific choice of $\su_0'$ in our numerical experiments,
namely
\begin{equation}
\label{eq:beta}
\frac{du}{dt}(0) = \frac{1 - 2\alpha}{2\alpha - \lambda + 1} f(\su_0),
\end{equation}
to better match the transient dynamics.

\subsection{Intuition and Wider Context}
\label{ssec:Link2}

\subsubsection{Idea Behind The Modified Equations}

In this subsection, we show that the scheme \eqref{eq:general_discrete} exhibits momentum, in the sense 
of approximating a momentum equation, but the size of the momentum term is on the order of the step size \(h\). 
To see this intuitively, we add and subtract \(\su_n - \su_{n-1}\) to the right hand size of \eqref{eq:general_discrete} then 
we can rearrange it to obtain
\[h \frac{\su_{n+1} - 2\su_n + \su_{n-1}}{h^2} + (1-\lambda) \frac{\su_n - \su_{n-1}}{h} = f(\su_n + a(\su_n - \su_{n-1})).\]
This can be seen as a second order central difference and first order backward difference discretization of the 
momentum equation
\[h \frac{d^2u}{dt^2} + (1-\lambda)\frac{du}{dt} = f(u)\]
noting that the second derivative term has size of order \(h\).

\subsubsection{Higher Order Modified Equations For HB}

We will now show that, for HB, we may derive higher order modified equations that are 
consistent with \eqref{eq:hb_discrete}. Taking the limit of these equations yields an 
operator that agrees with with our intuition for discretizing \eqref{eq:rgf}. To this end,
suppose \(\Phi \in C^\infty_b (\R^d, \R)\) and consider the ODE(s),
\begin{equation}
\label{eq:hb_highorder}
\sum_{k=1}^p \frac{h^{k-1}(1 + (-1)^k \lambda)}{k!} \frac{d^k u}{dt^k} = f(u)
\end{equation}
noting that \(p = 1\) gives \eqref{eq:rgf} and \(p=2\) gives \eqref{eq:visco}. Let \(u \in C^\infty ([0,\infty),\R^d)\)
be the solution to \eqref{eq:hb_highorder} and define \(u_n \coloneqq u(nh)\), \(u_n^{(k)} \coloneqq \frac{d^k u}{dt^k} (nh)\)
for \(n=0,1,2,\dots\) and \(k=1,2,\dots,p\). Taylor expanding yields
\[u_{n \pm 1} = u_n + \sum_{k=1}^p \frac{(\pm 1)^k h^k}{k!} u^{(k)}_n + h^{p+1} I^{\pm}_n\]
where
\[I^{\pm}_n = \frac{(\pm 1)^{p+1}}{p!} \int_0^1 (1-s)^p \frac{d^{p+1}u}{dt^{p+1}} ((n \pm s)h) ds.\]
Then 
\begin{align*}
u_{n+1} - u_n - \lambda(u_n - u_{n-1}) &= \sum_{k=1}^p \frac{h^k}{k!} u^{(k)}_n + \lambda \sum_{k=1}^p \frac{ (-1)^k h^k}{k!} u^{(k)}_n + h^{p+1}(I^+_n - \lambda I^-_n) \\
&= h \sum_{k=1}^p \frac{h^{k-1}(1 + (-1)^k \lambda)}{k!} u^{(k)}_n + h^{p+1}(I^+_n - \lambda I^-_n) \\
&= h f(u_n) + h^{p+1}(I^+_n - \lambda I^-_n)
\end{align*}
showing consistency to order \(p+1\). As is the case with \eqref{eq:visco} however, the \(I^\pm_n\) terms will 
be inversely proportional to powers of \(h\) hence global accuracy will not improve.

We now study the differential operator on the l.h.s. of \eqref{eq:hb_highorder} as \(p \rightarrow \infty\).
Define the sequence of differential operators \(T_p : C^\infty([0,\infty),\R^d) \to C^\infty([0,\infty),\R^d)\) by
\[T_p u = \sum_{k=1}^p \frac{h^{k-1}(1 + (-1)^k \lambda)}{k!} \frac{d^k u}{dt^k}, \quad \forall u \in C^\infty([0,\infty),\R^d).\]
Taking the Fourier transform yields
\[\mathcal{F}(T_p u)(\omega) = \sum_{k=1}^p \frac{h^{k-1}(1 + (-1)^k \lambda)(i \omega)^k}{k!} \mathcal{F}(u)(\omega)\]
where \(i = \sqrt{-1}\) denotes the imaginary unit. Suppose there is a limiting operator \(T_p \to T\) as \(p \to \infty\)
then taking the limit yields
\[\mathcal{F}(Tu)(\omega) = \frac{1}{h} (e^{ih\omega} + \lambda e^{-ih\omega} - \lambda -1) \mathcal{F}(u)(\omega).\]
Taking the inverse transform and using the convolution theorem, we obtain
\begin{align*}
(Tu)(t) &= \frac{1}{h} \mathcal{F}^{-1} ( e^{ i h \omega} + \lambda e^{- i h \omega} - \lambda - 1  )(t) * u(t) \\
&= \frac{1}{h} \left ( -(1+\lambda) \delta(t) + \lambda \delta(t + h) + \delta(t - h) \right) * u(t) \\
&= \frac{1}{h} \int_{-\infty}^\infty \left ( -(1+\lambda) \delta(t - \tau) + \lambda \delta(t - \tau + h) + \delta(t - \tau - h) \right)u(\tau) \: d\tau \\
&= \frac{1}{h} \left ( -(1+\lambda)u(t) + \lambda u(t -  h) + u(t +  h) \right ) \\
&= \frac{u(t+h) - u(t)}{h} - \lambda \left ( \frac{u(t) - u(t-h)}{h} \right )
\end{align*}
where \(\delta(\cdot)\) denotes the Dirac-delta distribution and we abuse notation by writing its action 
as an integral. The above calculation does not prove convergence of \(T_p\) to \(T\), but simply confirms 
our intuition that \eqref{eq:hb_discrete} is a forward and backward discretization of \eqref{eq:rgf}.

\begin{figure}[t]
    \centering
    \begin{subfigure}[b]{0.31\textwidth}
        \includegraphics[width=\textwidth]{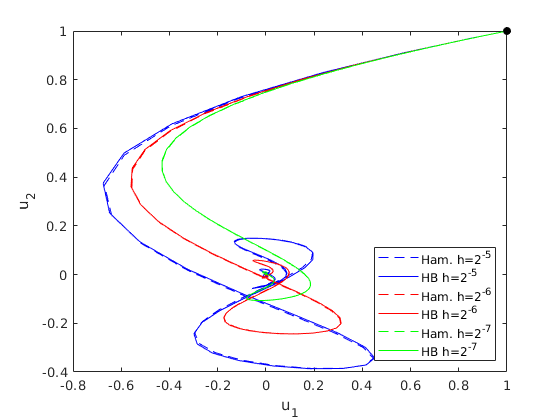}
        \caption{HB: \(\kappa = 5\)}
    \end{subfigure}
    ~ 
    \begin{subfigure}[b]{0.31\textwidth}
        \includegraphics[width=\textwidth]{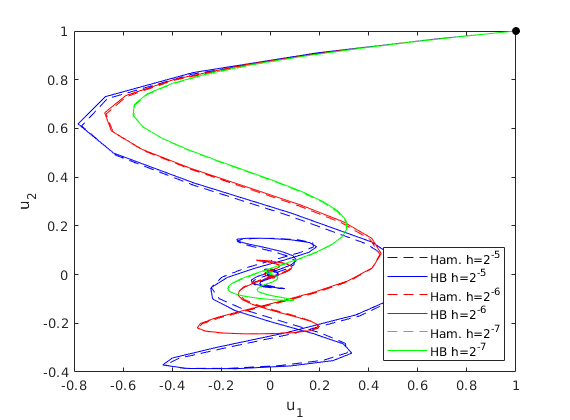}
        \caption{HB: \(\kappa = 10\)}
    \end{subfigure}
     ~ 
    \begin{subfigure}[b]{0.31\textwidth}
        \includegraphics[width=\textwidth]{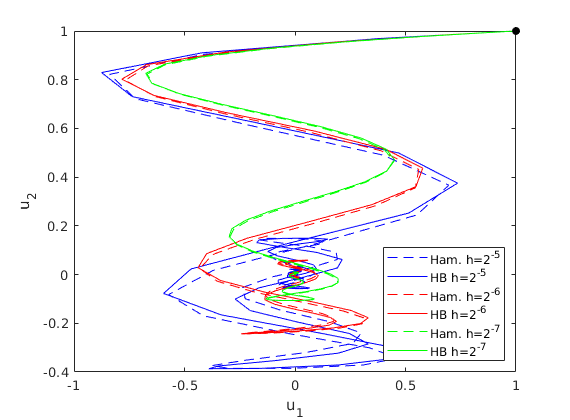}
        \caption{HB: \(\kappa = 20\)}
    \end{subfigure}

    \begin{subfigure}[b]{0.31\textwidth}
        \includegraphics[width=\textwidth]{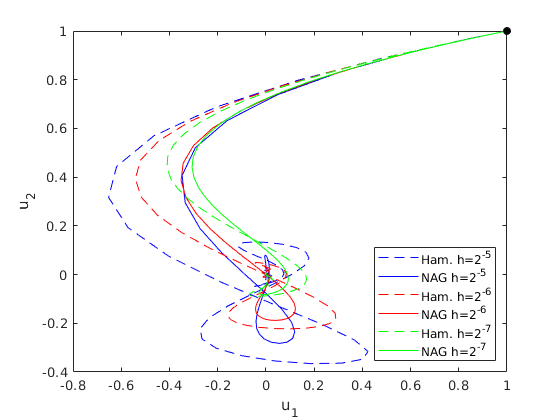}
        \caption{NAG: \(\kappa = 5\)}
    \end{subfigure}
    ~ 
    \begin{subfigure}[b]{0.31\textwidth}
        \includegraphics[width=\textwidth]{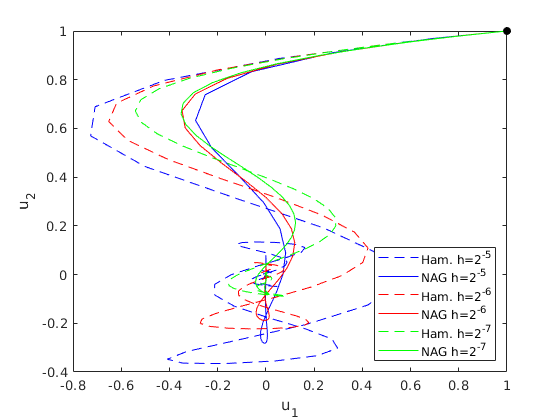}
        \caption{NAG: \(\kappa = 10\)}
    \end{subfigure}
     ~ 
    \begin{subfigure}[b]{0.31\textwidth}
        \includegraphics[width=\textwidth]{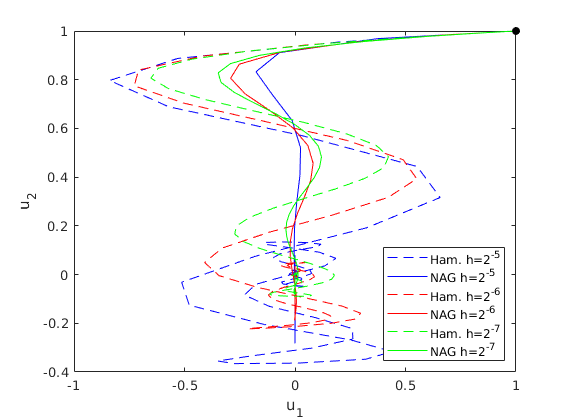}
        \caption{NAG: \(\kappa = 20\)}
    \end{subfigure}

    \caption{Comparison of trajectories for HB and NAG with the Hamiltonian dynamic \eqref{eq:visco} on the two-dimensional problem
    \(\Phi(u) = \frac{1}{2} \langle u, Q u \rangle\) with \(\lambda=0.9\) fixed. We vary the condition 
    number of \(Q\) as well as the learning rate \(h\).} 
    \label{fig:visco_traj}
\end{figure}

\begin{figure}[t]
    \centering
    \begin{subfigure}[b]{0.31\textwidth}
        \includegraphics[width=\textwidth]{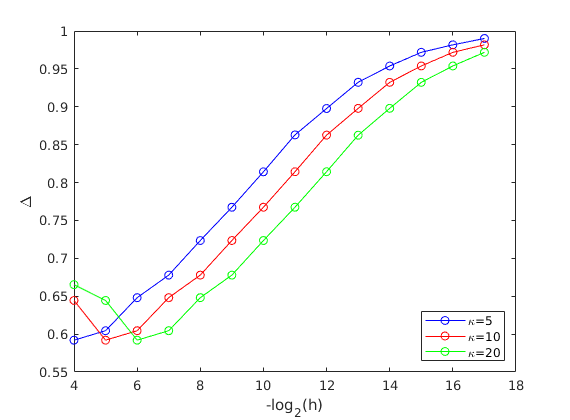}
        \caption{HB}
    \end{subfigure}
    ~ 
    \begin{subfigure}[b]{0.31\textwidth}
        \includegraphics[width=\textwidth]{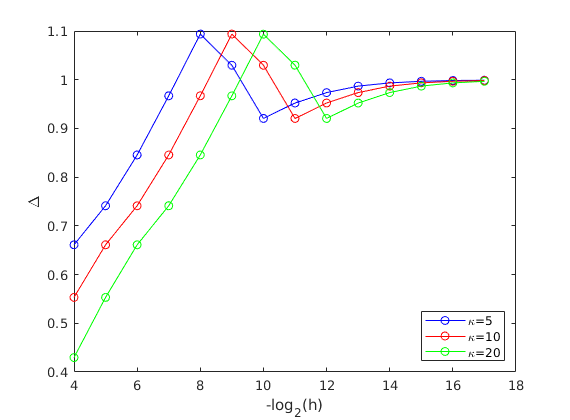}
        \caption{NAG}
    \end{subfigure}

    \caption{The numerical rate of convergence, as a function of the learning rate \(h\), of HB and NAG to the momentum equation \eqref{eq:visco}
    for the problem described in Figure \ref{fig:visco_traj}.} 
    \label{fig:visco_conv}
\end{figure}

\subsection{Numerical Illustration}
\label{ssec:N2}

Figure \ref{fig:visco_traj} shows trajectories of \eqref{eq:general_discrete} 
and \eqref{eq:visco} for different values of \(a\) and 
\(h\) on the two-dimensional problem \(\Phi(u) = \frac{1}{2} \langle u , Q u \rangle\),
varying the condition number of \(Q\). 
We make the specific choice of $\su_0'$ implied by the initial 
condition \eqref{eq:beta}. Figure \ref{fig:visco_conv} shows 
the numerical order of convergence as a function of \(h\), as defined in Section \ref{ssec:N1}, which is near 1, matching our theory. We note that the 
oscillations in HB are captured well by \eqref{eq:visco}, except for a slight shift when \(h\) and \(\kappa\) are large. This is due to 
our choice of initial condition which cancels the maximum number of terms in the Taylor expansion initially, but the overall rate of convergence remains
\(\mathcal{O}(h)\) due to Lemma \ref{lemma:visco_bound}. Other choices of $\su_0'$
also result in \(\mathcal{O}(h)\) convergence and can be picked on a case-by-case basis 
to obtain consistency with different qualitative phenomena of interest in the dynamics.
Note also that \(\alpha|_{a = \lambda} < \alpha|_{a = 0}\). As a result the transient
oscillations in \eqref{eq:visco} are more quickly damped in the NAG case than in 
the HB case; this is consistent with the numerical results. 
However panels (d)-(f) in Figure \ref{fig:gf_traj} show that \eqref{eq:visco} is not able to 
adequately capture the oscillations of NAG when \(h\) is relatively large. We 
leave for future work, the task of finding equations that are able to appropriately capture the 
oscillations of NAG in the large \(h\) regime.

\section{Invariant Manifold}
\label{sec:IM}

The key lessons of the previous two sections are that the momentum methods
approximate a rescaled gradient flow of the form \eqref{eq:GD} and a damped
Hamiltonian system of the form \eqref{eq:HAM}, with small mass $m$ which
scales with the learning rate, and constant damping $\gamma.$ 
Both approximations hold with the same
order of accuracy, in terms of the learning rate, and numerics demonstrate that the Hamiltonian system is
particularly useful in providing intuition for the transient regime of the
algorithm. In this section we link the two theorems from the two preceding
sections by showing that the Hamiltonian dynamics with small mass from section
\ref{sec:ME} has
an exponentially attractive invariant manifold on which the dynamics
is, to leading order, a gradient flow.
That gradient flow is a small, in terms of the learning rate,
perturbation of the time-rescaled gradient flow
from section \ref{sec:M}.

\subsection{Main Result}
\label{ssec:MR3}

Define 
\begin{equation}
\label{eq:define}
\sv_n \coloneqq (\su_n - \su_{n-1})/h
\end{equation}
noting that then \eqref{eq:general_discrete} becomes
\[\su_{n+1} = \su_n + h \lambda \sv_n + hf(\su_n + ha\sv_n)\]
and 
\[\sv_{n+1} = \frac{\su_{n+1} - \su_n}{h} = \lambda \sv_n + f(\su_n + ha\sv_n).\]
Hence we can re-write \eqref{eq:general_discrete} as 
\begin{align}
\label{eq:two_step_method}
\begin{split}
\su_{n+1} &= \su_n + h \lambda \sv_n + hf(\su_n + ha\sv_n) \\
\sv_{n+1} &= \lambda \sv_n + f(\su_n + ha \sv_n).
\end{split}
\end{align}

Note that if $h=0$ then \eqref{eq:two_step_method} shows that 
$\su_n=\su_0$ is constant in $n$, and that 
$\sv_n$ converges to $(1-\lambda)^{-1}f(\su_0).$ This suggests that,
for $h$ small, there is an invariant manifold which is a small perturbation of
the relation $\sv_n=\bar{\lambda}f(\su_n)$ 
and is representable as a graph. Motivated by this, we 
look for a function \(g: \Rd \to \Rd\) such that
the manifold
\begin{equation}
\label{eq:thisisit}
\sv=\bar{\lambda} f(\su) + hg(\su)
\end{equation}
is invariant for the dynamics of the numerical method:
\begin{equation}
\label{eq:highlight}
\sv_n = \bar{\lambda} f(\su_n) + hg(\su_n) \Longleftrightarrow \sv_{n+1} = \bar{\lambda} f(\su_{n+1}) + hg(\su_{n+1}).
\end{equation}

We will prove the existence of such a function $g$ by use of the contraction
mapping theorem to find fixed point of mapping $T$ defined in subsection \ref{ssec:I3}
below. We seek this fixed point in set $\Gamma$ which we now define:

\begin{definition}
\label{def:Gamma}
Let \(\gamma, \delta > 0\) be as in Lemmas \ref{lemma:gamma}, \ref{lemma:delta}. Define \(\Gamma \coloneqq \Gamma(\gamma,\delta)\) to be the closed 
subset of \(C(\Rd;\Rd)\) consisting of \(\gamma\)-bounded functions:
\[\|g\|_\Gamma \coloneqq \sup_{\xi \in \Rd} |g(\xi)| \leq \gamma, \quad \forall g \in \Gamma\]
that are \(\delta\)-Lipshitz:
\[|g(\xi) - g(\eta)| \leq \delta |\xi - \eta|, \quad \forall g \in \Gamma, \xi, \eta \in \Rd.\]
\end{definition}

\begin{theorem}
\label{thm:invman_existence}
Fix $\lambda \in (0,1).$ Suppose that $h$ is chosen small enough so that Assumption \ref{assump:h_small} holds. For \(n=0,1,2,\dots\), let \(\su_n\), \(\sv_n\) be the sequences given by \eqref{eq:two_step_method}. 
Then there is a \(\tau >0\) such that, for all \(h \in (0,\tau)\),  there is a 
unique \(g \in \Gamma\) such that \eqref{eq:highlight} holds. Furthermore,
\[|\sv_n - \bar{\lambda}f(\su_n) - hg(\su_n)| \leq (\lambda + h^2 \lambda \delta)^n |\sv_0 - \bar{\lambda}f(\su_0) - hg(\su_0)|\]
where \(\lambda + h^2 \lambda \delta < 1\).
\end{theorem}

The statement of Assumption \ref{assump:h_small}, and the proof of the preceding
theorem, are given in Appendix C. The assumption appears somewhat involved at first
glance but inspection reveals that it simply places an upper bound on the
learning rate $h,$ as detailed in Lemmas \ref{lemma:gamma}, \ref{lemma:delta}. 
The proof of the theorem rests on the Lemmas \ref{lemma:well_defined}, 
\ref{lemma:gamma_to_gamma} and \ref{lemma:T_contraction} which establish 
that the operator \(T\) is well-defined, maps \(\Gamma\) to \(\Gamma\), and 
is a contraction on \(\Gamma\). The operator $T$ is defined, and expressed
in a helpful form for the purposes of analysis, in the next subsection. 

In the next subsection
we obtain the leading order approximation for $g$, given in equation \eqref{eq:getg}.
Theorem \ref{thm:invman_existence} implies that the large-time dynamics
are governed by the dynamics on the invariant manifold. Substituting the 
leading order approximation for $g$ into the invariant manifold \eqref{eq:thisisit} 
and using this expression in the definition \eqref{eq:define}
shows that

\begin{subequations}
\label{eq:another}
\begin{align}
\label{eq:another_a}
\sv_n&=-(1-{\lambda})^{-1} \nabla \left ( \Phi(\su_n) + \frac{1}{2} h \bar{\lambda}(\bar{\lambda} - a) |\nabla \Phi(\su_n)|^2 \right ),\\
\label{eq:another_b}
\su_{n}&=\su_{n-1}
 - h(1-{\lambda})^{-1} \nabla \left ( \Phi(\su_n) + \frac{1}{2} h \bar{\lambda}(\bar{\lambda} - a) |\nabla \Phi(\su_n)|^2 \right ).
\end{align}
\end{subequations}
Setting 
\begin{equation}
\label{eq:mp_c}
c=\bar{\lambda} \left( \bar{\lambda}-a+\frac{1}{2} \right)
\end{equation}
we see that for large time the dynamics of momentum methods, including
HB and NAG, are approximately those of the modified gradient flow 
\begin{equation}
\label{eq:approxg2}
\dtfrac{u}
= - (1-{\lambda})^{-1} \nabla \Phi_h(u) 
\end{equation}
with
\begin{equation}
\label{eq:mp}
\Phi_h(u)=\Phi(u)+\frac12 hc|\nabla \Phi(u)|^2.
\end{equation}
To see this we proceed as follows. Note that from \eqref{eq:approxg2}
\[\frac{d^2u}{dt^2} = - \frac{1}{2} (1-\lambda)^{-2} \nabla |\nabla \Phi(u)|^2 + \mathcal{O}(h)\]
then Taylor expansion shows that, for $u_n=u(nh)$,
\begin{align*}
u_n &= u_{n-1} + h \dot{u}_n - \frac{h^2}{2} \ddot{u}_n + \mathcal{O}(h^3) \\
&=u_{n-1} - h \bar{\lambda} \left( \nabla \Phi(u_n) + \frac{1}{2}hc\nabla |\nabla \Phi(u_n)|^2 \right) + \frac{1}{4}h^2 \bar{\lambda}^2 \nabla |\nabla \Phi(u_n)|^2 + \mathcal{O}(h^3)
\end{align*}
where we have used that
$$Df(u)f(u)=\frac12 \nabla \left(|\nabla \Phi(u)|^2\right).$$
Choosing $c=\bar{\lambda}(\bar{\lambda}-a+1/2)$ we see that 
\begin{equation}
\label{eq:approxflow_taylor}
u_n=u_{n-1}
 - h(1-{\lambda})^{-1} \nabla \left ( \Phi(u_n) + \frac{1}{2} h \bar{\lambda}(\bar{\lambda} - a) |\nabla \Phi(u_n)|^2 \right ) + \mathcal{O}(h^3).
\end{equation}
Notice that comparison
of \eqref{eq:another_b} and \eqref{eq:approxflow_taylor} shows that, on the invariant manifold, the dynamics
are to \(\mathcal{O}(h^2)\) the same as the equation \eqref{eq:approxg2}; this is because the
truncation error between \eqref{eq:another_b} and \eqref{eq:approxflow_taylor} is \(\mathcal{O}(h^3)\). 

Thus we have proved:

\begin{theorem}
\label{thm:conv_to_modifiedgf}
Suppose that the conditions of Theorem \ref{thm:invman_existence} hold.
Then for initial data started on the invariant manifold and 
any \(T \geq 0\), there is a constant \(C = C(T) > 0\) such that
\[\sup_{0 \leq nh \leq T} |u_n - \su_n| \leq Ch^2,\]
where $u_n=u(nh)$ solves the modified equation \eqref{eq:approxg2}
with $c=\bar{\lambda}(\bar{\lambda}-a+1/2)$. 
\end{theorem}

\subsection{Intuition}
\label{ssec:I3}

We will define mapping \(T: C(\Rd;\Rd) \to C(\Rd;\Rd)\) via the equations 
\begin{equation}
\label{eq:two_step}
\begin{split}
p = \xi + h \lambda\bigl(\bar{\lambda} f(\xi) + hg(\xi)\bigr)  + 
hf\Bigl(\xi + ha\bigl(\bar{\lambda} f(\xi) + hg(\xi)\bigr)\Bigr) \\
\bar{\lambda} f(p)+h(Tg)(p) = \lambda \bigl(\bar{\lambda} f(\xi) + hg(\xi)\bigr) +
f\Bigl(\xi + ha\bigl(\bar{\lambda} f(\xi) + hg(\xi)\bigr)\Bigr). 
\end{split}
\end{equation}
A fixed point of the mapping $g \mapsto Tg$ will give function $g$ so that, under
\eqref{eq:two_step}, identity \eqref{eq:highlight} holds.
Later we will show that, for $g$ in $\Gamma$ and all $h$ sufficiently small,
 $\xi$ can be found from (\ref{eq:two_step}a)
for every $p$, and that thus (\ref{eq:two_step}b) defines a mapping from $g \in \Gamma$
into $Tg \in C(\Rd;\Rd).$ We will then show that, for $h$ sufficiently small,
$T: \Gamma \mapsto \Gamma$ is a contraction. 

For any \(g \in C(\Rd;\Rd)\) and \(\sun \in \Rd\) define
\begin{align}
\label{eq:wg}
w_g(\sun) &\coloneqq \bar{\lambda} f(\sun) + h g(\sun) \\
\label{eq:zg}
z_g(\sun) &\coloneqq \lambda w_g(\sun) + f\bigl(\sun + haw_g(\sun)\bigr).
\end{align}
With this notation the fixed point mapping \eqref{eq:two_step} for $g$ may be written
\begin{equation}
\label{eq:two_step2}
\begin{split}
p = \xi + hz_g(\xi),\\
\bar{\lambda} f(p)+h(Tg)(p) = z_g(\xi).
\end{split}
\end{equation}

Then, by Taylor expansion, 
\begin{align}
\label{eq:I1}
\begin{split}
f\Bigl(\sun + ha\bigl(\bar{\lambda} f(\sun) + hg(\sun)\bigr)\Bigr) &= f\bigl(\sun + haw_g(\sun)\bigr) \\
&= f(\sun) + ha \int_0^1  Df\bigl(\sun + shaw_g(\sun)\bigr)w_g(\sun) ds \\
&= f(\sun) + ha I^{(1)}_g (\sun)
\end{split}
\end{align}
where the last line defines \(I^{(1)}_g\). Similarly
\begin{align}
\label{eq:I2}
\begin{split}
f(p) &= f(\sun + hz_g(\sun)) \\
&= f(\sun) + h \int_0^1 Df\bigl(\sun + shz_g(\sun)\bigr)z_g(\sun)ds \\
&= f(\sun) + h I^{(2)}_g(\sun),
\end{split}
\end{align}
where the last line now defines \(I^{(2)}_g\).
Then (\ref{eq:two_step}b) becomes
\[\bar{\lambda}\bigl(f(\sun) + h I^{(2)}_g(\sun)\bigr) + h (Tg)(p) = \lambda \bar{\lambda} f(\sun) + h \lambda g(\sun) + f(\sun) + haI^{(1)}_g(\sun)\]
and we see that 
\[(Tg)(p) = \lambda g(\sun) + a I^{(1)}_g(\sun) - \bar{\lambda}I^{(2)}_g(\sun).\]
In this light, we can rewrite the defining equations \eqref{eq:two_step} for \(T\) as 
\begin{align}
\label{eq:p}
p &= \xi + hz_g(\xi), \\
\label{eq:Tg}
(Tg)(p) &= \lambda g(\xi) + a I^{(1)}_g(\xi) - \bar{\lambda}I^{(2)}_g(\xi). 
\end{align}
for any \(\xi \in \Rd\).

Perusal of the above definitions reveals that, to leading order in $h$,
$$w_g(\xi)=z_g(\xi)=\bar{\lambda}f(\xi), I^{(1)}_g(\xi)=I^{(2)}_g(\xi)=\bar{\lambda}
Df(\xi)f(\xi).$$
Thus setting $h=0$ in \eqref{eq:p}, \eqref{eq:Tg} shows that, to leading order
in $h$,
\begin{equation}
\label{eq:getg}
g(p) = \bar{\lambda}^2(a-\bar{\lambda})Df(p)f(p).
\end{equation}
Note that since \(f(p) = - \nabla \Phi(p)\), \(Df\) is the negative Hessian of \(\Phi\) and is thus symmetric.
Hence we can write $g$ in gradient form, leading to
\begin{equation}
\label{eq:getg2}
g(p) = \frac12 \bar{\lambda}^2(a-\bar{\lambda})\nabla\bigl(|\nabla \Phi(p)|^2\bigr).
\end{equation}

\begin{remark}
This modified potential \eqref{eq:mp} also arises in the construction of Lyapunov functions
for the one-stage theta method  -- see Corollary 5.6.2 in \cite{stuart1998dynamical}.
\end{remark}

\begin{figure}[t]
    \centering
    \begin{subfigure}[b]{0.31\textwidth}
        \includegraphics[width=\textwidth]{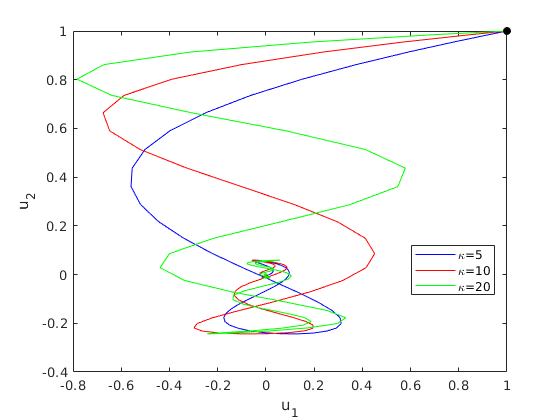}
        \caption{HB: \(\su_n\) given by \eqref{eq:two_step_method}}
    \end{subfigure}
    ~ 
    \begin{subfigure}[b]{0.31\textwidth}
        \includegraphics[width=\textwidth]{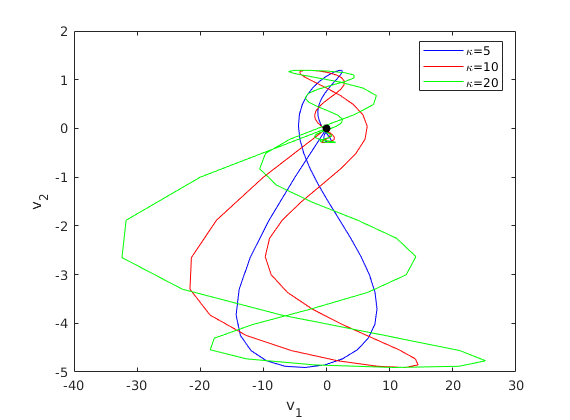}
        \caption{HB: \(\sv_n\) given by \eqref{eq:two_step_method}}
    \end{subfigure}
     ~ 
    \begin{subfigure}[b]{0.31\textwidth}
        \includegraphics[width=\textwidth]{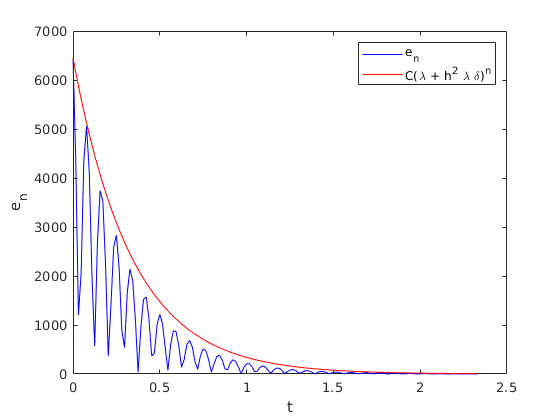}
        \caption{HB: \(\mathsf{e}_n\) given by \eqref{eq:ima2}}
    \end{subfigure}

     \begin{subfigure}[b]{0.31\textwidth}
        \includegraphics[width=\textwidth]{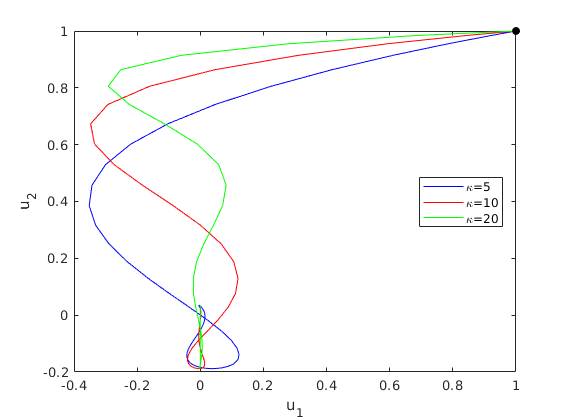}
        \caption{NAG: \(\su_n\) given by \eqref{eq:two_step_method}}
    \end{subfigure}
    ~ 
    \begin{subfigure}[b]{0.31\textwidth}
        \includegraphics[width=\textwidth]{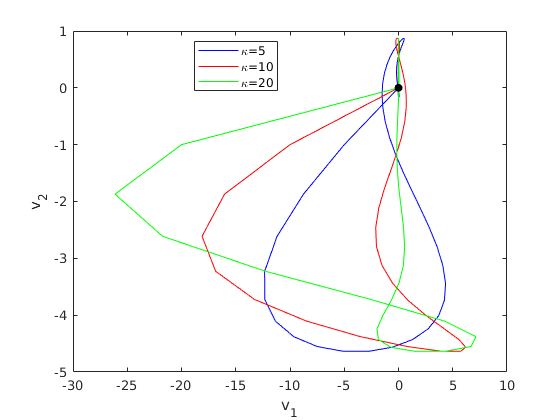}
        \caption{NAG: \(\sv_n\) given by \eqref{eq:two_step_method}}
    \end{subfigure}
     ~ 
    \begin{subfigure}[b]{0.31\textwidth}
        \includegraphics[width=\textwidth]{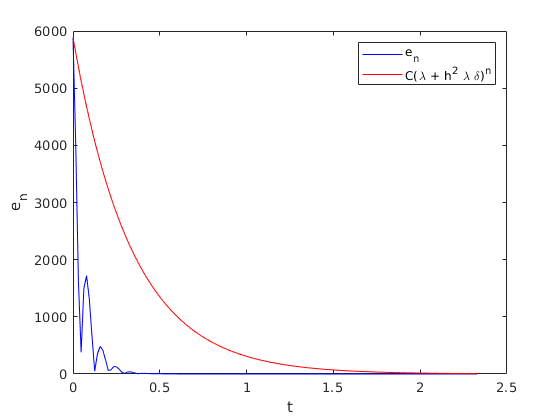}
        \caption{NAG: \(\mathsf{e}_n\) given by \eqref{eq:ima2}}
    \end{subfigure}
    \caption{Invariant manifold for HB and NAG with \(h=2^{-6}\) and \(\lambda = 0.9\) on the two-dimensional problem \(\Phi(u) = \frac{1}{2} \langle u, Q u \rangle\), varying the condition number of \(Q\). Panels (c), (f) show the distance from the invariant manifold for the largest condition number \(\kappa = 20\).} 
    \label{fig:invman}
\end{figure}

\subsection{Numerical Illustration}
\label{ssec:NI}

In Figure \ref{fig:invman} panels (a),(b),(d),(e), we plot the components $\su_n$ and $\sv_n$
found by solving \eqref{eq:two_step_method} with initial conditions $\su_0=(1,1)^T$ and 
$\sv_n=(0,0)^T$ in the case where  \(\Phi(u) = \frac{1}{2} \langle u, Q u \rangle\). 
These initial conditions correspond to initializing the map off the invariant manifold. 
To leading order in $h$ the invariant manifold is given by 
(see equation \eqref{eq:another})
\begin{equation}
\label{eq:ima1}
v= - (1-{\lambda})^{-1} \nabla \left ( \Phi(u) + \frac{1}{2} h \bar{\lambda}(\bar{\lambda} - a) |\nabla \Phi(u)|^2 \right ).
\end{equation}
To measure the distance of the trajectory shown in panels (a),(b),(d),(e) from the
invariant manifold we define
\begin{equation}
\label{eq:ima2}
\se_n= \left | \sv_n+ (1-{\lambda})^{-1} \nabla \left ( \Phi(\su_n) + \frac{1}{2} h \bar{\lambda}(\bar{\lambda} - a) |\nabla \Phi(\su_n)|^2 \right ) \right |.
\end{equation}
Panels (c),(f) show the evolution of $\se_n$ as well as the (approximate) bound on it found
from substituting the leading order approximation of $g$ into the following
upper bound from Theorem \ref{thm:invman_existence}: 
$$(\lambda + h^2 \lambda \delta)^n |\sv_0 - \bar{\lambda}f(\su_0) - hg(\su_0)|.$$

\section{Deep Learning Example}
\label{ssec:DL}

Our theory is developed under quite restrictive assumptions, in order
to keep the proofs relatively simple and to allow a clearer conceptual
development. The purpose of the numerical experiments in this section
is twofold: firstly to demonstrate that our theory sheds light on a stochastic
version of gradient descent applied, furthermore,
to a setting in which the objective
function does not satisfy the global assumptions which 
facilitate our analysis; and second to show that methods implemented
as we use them here (with learning-rate independent momentum, fixed
at every step of the iteration) can out-perform other choices on
specific problems.

Our numerical experiments in this section are
undertaken with in the context of the example 
given in \cite{importanceofinitmom}. We train a deep autoencoder, using 
the architecture of \cite{pretraining} on the MNIST dataset \cite{mnist}. Since our work is 
concerned only with optimization and not generalization, we present our results only on the
training set of 60,000 images and ignore the testing set. We fix an initialization of the 
autoencoder following \cite{xavier} and use it to test every optimization method. 
Furthermore, we fix a batch size of 200 and train for 500 epochs, not shuffling the data set during training so that 
each method sees the same realization of the noise. We use the mean-squared error as 
our loss function.

\begin{figure}[t]
\begin{center}
\begin{tabular}{ |l|c|c|c|c|c|c|c| } 
\hline
  & \(h=2^0\) & \(h=2^{-1}\) & \(h=2^{-2}\) & \(h=2^{-3}\) & \(h=2^{-4}\) & \(h=2^{-5}\) & \(h=2^{-6}\) \\
\specialrule{.1em}{.05em}{.05em} 
GF & n/a & 4.3948 & 4.5954 & 5.6769 & \bf{7.0049} & \bf{8.6468} & \bf{10.6548} \\
\hline
HB & 3.6775 & 4.0157 & 4.5429 & 5.6447 & 7.0720 & 8.7070 & 10.6848 \\
\hline
NAG & \bf{3.2808} &   \bf{3.7166} &  \bf{4.4579} &   \bf{5.6087} &  7.0557 & 8.6987 &  10.6814 \\
\specialrule{.1em}{.05em}{.05em} 
Wilson & 6.7395 &   7.5177 &    8.3491 &   9.2543 &   10.2761 &   11.3776 &   12.4123 \\
\hline
HB-\(\mu\) & 5.7099 &  6.6146 &  7.6202 &  8.6629 &   9.7838 &  11.0039 &  12.1743 \\
\hline
NAG-\(\mu\) & \bf{5.6867} & \bf{6.6033} &  \bf{7.6131} & \bf{8.6556} &  \bf{9.7783} &  \bf{11.0015} &  \bf{12.1738} \\ 
\hline
\end{tabular}
\end{center}
\caption{Final training errors for the autoencoder on  MNIST for six 
training methods over different learning rates. GF refers to equation \eqref{eq:rgf_disc} while
HB and NAG to \eqref{eq:general_discrete} all with fixed \(\lambda = 0.9\).}
\label{fig:final_training}
\end{figure}

\begin{figure}[t]
    \centering
    \begin{subfigure}[b]{0.45\textwidth}
        \includegraphics[width=\textwidth]{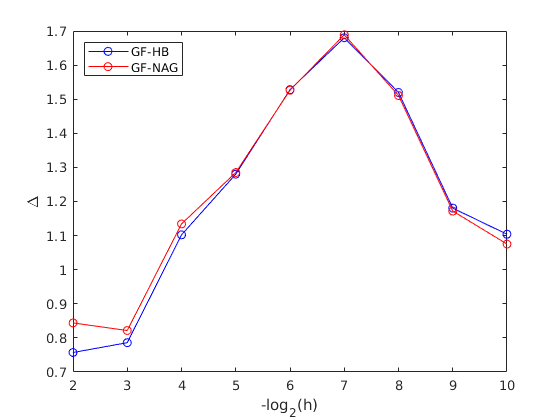}
        \caption{HB, NAG to \eqref{eq:rgf_disc}}
    \end{subfigure}
    ~ 
    \begin{subfigure}[b]{0.45\textwidth}
        \includegraphics[width=\textwidth]{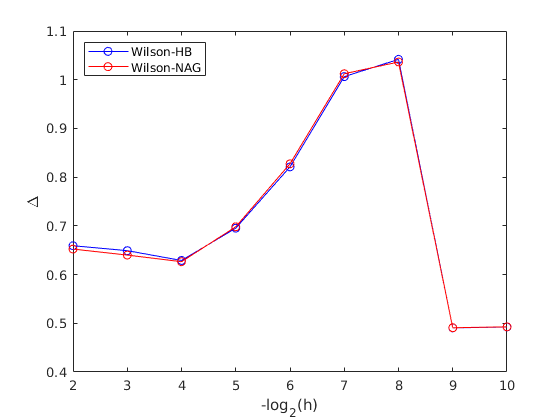}
        \caption{HB-\(\mu\), NAG-\(\mu\) to \eqref{eq:wilson_disc}}
    \end{subfigure}

    \caption{The numerical rate of convergence for the parameters of
    the autoencoder, as a function of the learning rate h, of HB
    and NAG to \eqref{eq:rgf_disc} (a), as well as of HB-\(\mu\) and NAG-\(\mu\) to \eqref{eq:wilson_disc}
    (b).} 
    \label{fig:autoencoder_conv}
\end{figure}

We compare HB and NAG given by \eqref{eq:general_discrete} to the re-scaled gradient 
flow \eqref{eq:rgf} which we discretize in the standard way to yield the numerical 
method
\begin{equation}
\label{eq:rgf_disc}
\su_{n+1} = \su_n - \frac{h}{(1-\lambda)} \nabla \Phi(\su_n),
\end{equation}
hence the momentum term \(\lambda\) only acts to re-scale the learning rate.
We do not test against equation \eqref{eq:visco} because, to discretize it 
faithfully, we would need to use a time-step much lower than \(h\) (because 
\eqref{eq:visco} contains a term of order \(h\)), but doing so would mean that
we need to train for many more epochs compared to HB and NAG so that the same 
final time is reached. This, in turn, implies that the methods would see
different realization of the noise. Thus, to compare them well, we would need
to perform a Monte Carlo simulation, however, since we do not state any 
of our results in a stochastic setting, we leave this for future work. 

We also compare our results to those of \cite{wilsoneq} which analyze HB and 
NAG in the setting where \(\Phi\) is \(\mu\)-strongly convex and \(\lambda\) is
given by \eqref{eq:lambda_mu} that is
\[
\lambda  = \frac{1 - \sqrt{\mu h}}{1 + \sqrt{\mu h}}.
\]
They obtain the limiting equation 
\[\ddot{u} + 2\sqrt{\mu} \dot{u} + \nabla \Phi(u) = 0\]
which we discretize via a split-step method to yield
\begin{align}
\label{eq:wilson_disc}
\begin{split}
\su_{n+1} &= \su_n + \frac{1}{2\sqrt{\mu}} \left ( 1- e^{-2\sqrt{ \mu h}} \right ) \sv_n \\
\sv_{n+1} &= e^{-2\sqrt{ \mu h}} \sv_n - \sqrt{h} \nabla \Phi(\su_{n+1})
\end{split}
\end{align}
where we have mapped the the time-step \(h\) in HB and 
NAG to \(\sqrt{h}\) as in done in \cite{wilsoneq}. We choose this discretization 
because it allows us to directly solve for the linear parts of the ODE (in 
the enlarged state-space), yielding a more accurate approximation than the 
forward-Euler method used to obtain \eqref{eq:rgf_disc}. A detailed derivation is given in Appendix D.
We will refer to the method in equation \eqref{eq:wilson_disc} as Wilson. 
Further we refer to equation \eqref{eq:general_discrete} 
with \(\lambda\) given by \eqref{eq:lambda_mu} and \(a=0\) as HB-\(\mu\) and 
equation \eqref{eq:general_discrete} with \(\lambda\) given by \eqref{eq:lambda_mu} and \(a=\lambda\) as NAG-\(\mu\).  
Since deep neural networks are not strongly convex, there is no single
optimal choice of $\mu$; we simply set \(\mu=1\) in our experiments. 

Figure \ref{fig:final_training} gives the final training errors for each 
method for several learning rates. We were unable to train the autoencoder 
using \eqref{eq:rgf_disc} with \(h=1\) since \(\lambda=0.9\) implies 
an effective learning rate of \(10\) for which the system blows up.
In general, 
NAG is the best performing method for relatively 
large \(h\) which is an observation that is consistently made in the 
deep learning literature. Further, we note that as the learning rate 
decreases, the final errors become closer indicating convergence to the 
appropriate limiting equations. Figure \ref{fig:final_training} showcases the practical effectiveness of momentum methods as they 
provide a way of discretizing the gradient flow \eqref{eq:GD} with a large effective learning 
rate that forward Euler cannot accommodate. From this perspective, we can view momentum
methods as providing a more stable discretization to gradient flows in a 
manner illustrated by \eqref{eq:approxg2}.
Such a viewpoint informs the works \cite{scieur2017integration,betancourt2018symplectic,runge-kuttadisc}.

To further illustrate the point of convergence to the limiting equation, we 
compute the numerical rate of convergence, defined in Section \ref{ssec:N1}, 
as a function of \(h\)
for the neural network parameters between \eqref{eq:rgf_disc} and HB and NAG
as well as between \eqref{eq:wilson_disc} and HB-\(\mu\) and NAG-\(\mu\). Figure 
\ref{fig:autoencoder_conv} gives the results. We note that this rate is around 
1 as predicted by our theory while the rate for \eqref{eq:wilson_disc} is around 
0.5 which is also consistent with the theory in \cite{wilsoneq}.

\section{Conclusion}
\label{sec:C}

Together, equations \eqref{eq:rgf}, \eqref{eq:visco} and \eqref{eq:approxg2}
describe the dynamical systems which are approximated by momentum methods, when
implemented with fixed momentum,
in a manner made precise by the four theorems in this paper. The insight obtained
from these theorems sheds light on how momentum methods perform optimization tasks.


\acks{Both authors are supported, in part, by the US National Science Foundation (NSF)
grant DMS 1818977, the US Office of Naval Research (ONR) grant N00014-17-1-2079, and
the US Army Research Office (ARO) grant W911NF-12-2-0022. Both authors are also
grateful to the anonymous reviewers for their invaluable suggestions which have 
helped to significantly strengthen this work.}


\newpage

\appendix

\section*{Appendix A}
\label{sec:AA}

\begin{proof}[of Theorem \ref{thm:conv_to_gf}]
Taylor expanding yields
\[u_{n+1} = u_n + h \bar{\lambda} f(u_n) + \Bigo(h^2)\]
and
\[u_n = u_{n-1} + h \bar{\lambda} f(u_n) + \Bigo(h^2).\]
Hence
\[(1+\lambda)u_n - \lambda u_{n-1} = u_n + h \lambda \bar{\lambda} f(u_n) + \Bigo(h^2).\]
Subtracting the third identity from the first, we find that
\[u_{n+1} - \left ( (1+\lambda) u_n - \lambda u_{n-1} \right ) = hf(u_n) + \Bigo(h^2)\]
by noting \(\bar{\lambda} - \bar{\lambda} \lambda = 1\).
Similarly,
\[a(u_n - u_{n-1}) = h a \bar{\lambda} f(u_n) + \Bigo(h^2)\]
hence Taylor expanding yields
\begin{align*}
f(u_n + a(u_n-u_{n-1})) &= f(u_n) + a Df(u_n)(u_n - u_{n-1})  \\
&\,\,\,\,\, + a^2 \int_0^1 (1-s) D^2f(u_n + sa(u_n-u_{n-1}))[u_n-u_{n-1}]^2 ds \\
&= f(u_n) + h a \bar{\lambda} Df(u_n)f(u_n) + \Bigo(h^2).
\end{align*}
From this, we conclude that
\[hf(u_n + a(u_n-u_{n-1})) = hf(u_n) + \Bigo(h^2)\]
hence
\[u_{n+1} = (1+\lambda)u_n - \lambda u_{n-1} + hf(u_n + a(u_n-u_{n-1})) + \Bigo(h^2).\]
Define the error \(e_n \coloneqq u_n - \mathsf{u}_n\) then
\begin{align*}
e_{n+1} &= (1 + \lambda)e_n - \lambda e_{n-1} + h \left ( f(u_n + a(u_n - u_{n-1})) - f(\mathsf{u}_n + a(\mathsf{u}_n - \mathsf{u}_{n-1})) \right ) + \Bigo(h^2) \\
&= (1 + \lambda)e_n - \lambda e_{n-1} + h \mathsf{M}_n((1+a)e_n - a e_{n-1}) + \Bigo(h^2)
\end{align*}
where, from the mean value theorem, we have
\[\mathsf{M}_n = \int_0^1 Df \Bigl ( s \bigl( u_n + a(u_n-u_{n-1}) \bigr) + \bigl( 1-s \bigr) \bigl( \mathsf{u}_n + a(\mathsf{u}_n-\mathsf{u}_{n-1}) \bigr) \Bigr ) ds.\]
Now define the concatenation \(E_{n+1} \coloneqq [e_{n+1}, e_n] \in \R^{2d}\) then
\[E_{n+1} = A^{(\lambda)} E_n + h A^{(a)}_n E_n + \Bigo(h^2)\]
where \(A^{(\lambda)}, A^{(a)}_n \in \R^{2d \times 2d}\) are the block matrices
\[A^{(\lambda)} \coloneqq
\begin{bmatrix}
(1+\lambda)I & - \lambda I \\
I & 0I
\end{bmatrix}, \quad
A^{(a)}_n \coloneqq
\begin{bmatrix}
(1+a) \mathsf{M}_n & - a\mathsf{M}_n \\
0I & 0 I
\end{bmatrix}\]
with \(I \in \R^{d \times d}\) the identity. We note that \(A^{(\lambda)}\) has minimal polynomial
\[\mu_{A^{(\lambda)}}(z) = (z-1)(z-\lambda)\]
and is hence diagonalizable. Thus there is a norm on \(\|\cdot\|\) on \(\R^{2d}\) such that its 
induced matrix norm \(\|\cdot\|_m\) satifies \(\|A^{(\lambda)}\|_m = \rho(A^{(\lambda)})\)
where \(\rho : \R^{2d \times 2d} \to \R_+\) maps a matrix to its spectral radius. Hence, since
\(\lambda \in (0,1)\), we have \(\|A^{(\lambda)}\|_m = 1\). Thus
\[\|E_{n+1}\| \leq (1 + h \|A^{(a)}_n\|_m) \|E_n\| + \Bigo(h^2).\]
Then, by finite dimensional norm equivalence, there is a constant \(\alpha > 0\), independent of \(h\), such that
\begin{align*}
\|A^{(a)}_n\|_m &\leq \alpha\left  \| \begin{bmatrix} 1+a & -a \\ 0 & 0 \end{bmatrix} \otimes \mathsf{M}_n \right \|_2 \\
&= \alpha \sqrt{2a^2 + 2a + 1} \|\mathsf{M}_n\|_2
\end{align*}
where \(\|\cdot\|_2\) denotes the spectral 2-norm. Using Assumption \ref{assump:one}, we have 
\[\|\mathsf{M}_n\|_2 \leq B_1\]
thus, letting \(c \coloneqq \alpha \sqrt{2a^2 + 2a + 1}  B_1\), we find
\[\|E_{n+1}\| \leq (1+hc)\|E_n\| + \Bigo(h^2).\]
Then, by Gr{\"o}nwall lemma,
\begin{align*}
\|E_{n+1}\| &\leq (1+hc)^n \|E_1\|_n + \frac{(1+hc)^{n+1}-1}{ch} \Bigo(h^2) \\
&= (1+hc)^n \|E_1\|_n + \Bigo(h)
\end{align*}
noting that the constant in the \(\Bigo(h)\) term is bounded above in terms of \(T\), but independently of \(h\).
Finally, we check the initial condition
\[E_1 = 
\begin{bmatrix}
u_1 - \mathsf{u}_1 \\
u_0 - \mathsf{u}_0
\end{bmatrix} = 
\begin{bmatrix}
h(\bar{\lambda}-1)f(\mathsf{u}_0) + \Bigo(h^2) \\
0
\end{bmatrix} = \Bigo(h)\]
as desired.
\end{proof}

\section*{Appendix B}
\label{sec:AB}

\begin{proof}[of Theorem \ref{thm:conv_to_visco}]
Taylor expanding yields 
\[u_{n \pm 1} = u_n \pm h \dot{u}_n + \frac{h^2}{2} \ddot{u}_n \pm \frac{h^3}{2} I^{\pm}_n\]
where
\[I^\pm_n = \int_0^1 (1-s)^2 \dddot{u}((n \pm s)h)ds.\]
Then using equation \eqref{eq:visco}
\begin{align}
\label{eq:app_b_thm_eq1}
\begin{split}
u_{n+1} - u_n - \lambda (u_n - u_{n-1}) &= h(1-\lambda)\dot{u}_n + \frac{h^2}{2} (1+\lambda) \ddot{u}_n + \frac{h^3}{2}(I^+_n - \lambda I^-_n) \\
&= hf(u_n) + h^2a(1-\lambda) \ddot{u}_n + \frac{h^3}{2}(I^+_n - \lambda I^-_n).
\end{split}
\end{align}
Similarly
\[a(u_n - u_{n-1}) = ha\dot{u}_n - \frac{h^2}{2}a \ddot{u}_n + \frac{h^3}{2}aI^-_n\]
hence
\[f(u_n + a(u_n-u_{n-1})) = f(u_n) + haDf(u_n) \dot{u}_n - Df(u_n) \left ( \frac{h^2}{2}a \ddot{u}_n - \frac{h^3}{2}a I_n^- \right ) + I^f_n\]
where 
\[I^f_n = a^2 \int_0^1 (1-s) D^2f(u_n + sa(u_n-u_{n-1}))[u_n - u_{n-1}]^2 ds.\]
Differentiating \eqref{eq:visco} yields
\[h \alpha \frac{d^3u}{dt^3} + (1-\lambda)\frac{d^2u}{dt^2} = Df(u) \frac{du}{dt}\]
hence
\begin{align*}
hf(u_n + a(u_n-u_{n-1})) &= hf(u_n) + h^2 a \left ( h \alpha \dddot{u}_n + (1-\lambda) \ddot{u}_n  \right ) - Df(u_n) \left ( \frac{h^3}{2}a \ddot{u}_n - \frac{h^4}{2}a I_n^- \right ) + hI^f_n  \\
&=hf(u_n) + h^2 a(1-\lambda) \ddot{u}_n + h^3 a \alpha \dddot{u}_n - Df(u_n) \left ( \frac{h^3}{2}a \ddot{u}_n - \frac{h^4}{2}a I_n^- \right ) + hI^f_n.
\end{align*}
Rearranging this we obtain an expression for \(hf(u_n)\) which we plug into equation \eqref{eq:app_b_thm_eq1}
to yield
\[u_{n+1} - u_n - \lambda (u_n - u_{n-1}) = hf(u_n + a(u_n-u_{n-1})) + \text{LT}_n\]
where 
\[\text{LT}_n = \underbrace{\frac{h^3}{2} (I^+_n - \lambda I^-_n)}_{\mathcal{O} \left( h \text{exp} \left (-\frac{(1-\lambda)}{2 \alpha } n \right) \right)} - \underbrace{h^3 a \alpha \dddot{u}_n}_{\mathcal{O} \left( h \text{exp} \left (-\frac{(1-\lambda)}{2 \alpha } n \right) \right)} + \underbrace{Df(u_n) \left ( \frac{h^3}{2}a \ddot{u}_n - \frac{h^4}{2}a I_n^- \right )}_{\mathcal{O}(h^2)} - \underbrace{hI^f_n}_{\mathcal{O}(h^3)}.\]
The bounds (in braces) on the four terms above follow from employing Assumption \ref{assump:one} and Lemma \ref{lemma:visco_bound}. From them we deduce the existence of constants \(K_1,K_2 > 0\) independent of \(h\)
such that
\[|\text{LT}_n| \leq h K_1 \text{exp} \left ( -\frac{(1-\lambda)}{2 \alpha } n \right ) + h^2 K_2.\]
We proceed similarly to the proof of Theorem \ref{thm:conv_to_gf}, but with a different truncation error structure, and find the error satsifies
\[\|E_{n+1}\| \leq (1+hc)\|E_n\| + h K_1 \text{exp} \left ( -\frac{(1-\lambda)}{2 \alpha } n \right ) + h^2 K_2\]
where we abuse notation and continue to write \(K_1,K_2\) when, in fact, the constants have changed by use of finite-dimensional
norm equivalence.
Define \(K_3 \coloneqq K_2/c\) then summing this error, we find 
\begin{align*}
\|E_{n+1}\| &\leq (1+hc)^n \|E_1\| + h K_3 ((1 + hc)^{n+1} - 1) + h K_1 \sum_{j=0}^n (1+hc)^j \text{exp} \left ( - \frac{(1-\lambda)}{2 \alpha} (n-j) \right) \\
&= (1+hc)^n \|E_1\| + hK_3 ((1 + hc)^{n+1} - 1) + hK_1 S_n.
\end{align*}
where 
\[S_n = \text{exp} \left ( - \frac{(1-\lambda)}{2\alpha}n \right ) \left ( \frac{(1+hc)^{n+1} \text{exp} \left ( \frac{(1-\lambda)}{2\alpha} (n+1) \right ) - 1}{(1+hc) \text{exp} \left ( \frac{1-\lambda}{2\alpha} \right) - 1} \right ).\]
Let \(T = nh\) then
\begin{align*}
S_n &\leq \frac{(1+hc)^{n+1} \text{exp} \left ( \frac{1-\lambda}{2\alpha} \right )}{(1+hc) \text{exp} \left ( \frac{1-\lambda}{2\alpha} \right) - 1} \\
&\leq \frac{ 2 \text{exp} \left ( cT + \frac{1-\lambda}{2\alpha} \right )}{\text{exp} \left ( \frac{1-\lambda}{2\alpha}\right ) - 1}
\end{align*}
From this we deduce that 
\[\|E_{n+1}\| \leq (1+hc)^n \|E_1\| + \mathcal{O}(h)\]
noting that the constant in the \(\Bigo(h)\) term is bounded above in terms of \(T\), but independently of \(h\). For the initial condition, we check
\[u_1 - \su_1 = h (\su_0' - f(\su_0)) + \frac{h^2}{2} \ddot{u}_0 + \frac{h^3}{2} I^+_0\]
which is \(\mathcal{O}(h)\) by Lemma \ref{lemma:visco_bound}. Putting the bounds together 
we obtain
\[\sup_{0 \le nh \leq T} \|E_n\| \le C(T)h.\]

\end{proof}

\begin{lemma}
\label{lemma:visco_bound}
Suppose Assumption \ref{assump:one} holds and let \(u \in C^3([0,\infty);\Rd)\) be the solution to
\begin{align*}
&h \alpha \frac{d^2u}{dt^2} + (1-\lambda)\frac{du}{dt} = f(u) \\
&u(0) = \su_0, \quad \frac{du}{dt}(0) = \sv_0
\end{align*}
for some \(\su_0, \sv_0 \in \Rd\) and \(\alpha > 0\) independent of \(h\). Suppose \(h \leq (1-\lambda)^2/2 \alpha B_1\) then
there are constants \(C^{(1)}, C^{(2)}_1, C^{(2)}_2, C^{(3)}_1, C^{(3)}_2 > 0\) independent of \(h\) such that for any \(t \in [0,\infty)\),
\begin{align*}
|\dot{u}(t)| &\leq C^{(1)}, \\
|\ddot{u}(t)| &\leq \frac{C^{(2)}_1}{h} \text{exp} \left ( -\frac{(1-\lambda)}{2h\alpha} t \right ) + C^{(2)}_2, \\
|\dddot{u}(t)| &\leq \frac{C^{(3)}_1}{h^2} \text{exp} \left ( -\frac{(1-\lambda)}{2h\alpha} t \right ) + C^{(3)}_2.
\end{align*}
\end{lemma}

One readily verifies that the result of Lemma \ref{lemma:visco_bound} is tight by considering the one-dimensional
case with \(f(u) = - u\). This implies that the result of Theorem 
\ref{thm:conv_to_visco} cannot be improved without further assumptions.

\begin{proof}[of Lemma \ref{lemma:visco_bound}] 
Define \(v \coloneqq \dot{u}\) then
\[\dot{v} = - \frac{1}{h \alpha} \left ( (1-\lambda) v - f(u) \right ).\]
Define \(w \coloneqq (1-\lambda)v - f(u)\) hence \(\dot{v} = -(1/h\alpha)w\) and 
\(\dot{u} = v = \bar{\lambda}(w + f(u))\). Thus 
\begin{align*}
\dot{w} &= (1-\lambda)\dot{v} - Df(u)\dot{u} \\
&= - \frac{(1-\lambda)}{h \alpha} w - Df(u)(\bar{\lambda}(w + f(u))). 
\end{align*}
Hence we find
\begin{align*}
\frac{1}{2} \frac{d}{dt} |w|^2 &=  - \frac{(1-\lambda)}{h\alpha} |w|^2 - \bar{\lambda} \langle w, Df(u) w \rangle - \bar{\lambda} \langle w, Df(u)f(u) \rangle \\
&\leq - \frac{(1-\lambda)}{h\alpha} |w|^2 + \bar{\lambda} |\langle w, Df(u) w \rangle| + \bar{\lambda} |\langle w, Df(u)f(u) \rangle| \\
&\leq - \frac{(1-\lambda)}{h\alpha} |w|^2 + \bar{\lambda} B_1 |w|^2 + \bar{\lambda} B_0 B_1 |w| \\
&\leq - \frac{(1-\lambda)}{h\alpha} |w|^2 + \frac{(1-\lambda)}{2h\alpha} |w|^2 + \bar{\lambda} B_0 B_1 |w| \\
&= -\frac{(1-\lambda)}{2h\alpha} |w|^2  + \bar{\lambda}B_0 B_1 |w|
\end{align*}
by noting that our assumption \(h \leq (1-\lambda)^2 / 2 \alpha B_1\) implies \(\bar{\lambda}B_1 \leq (1-\lambda) / 2 h \alpha\). Hence 
\[\frac{d}{dt} |w| \leq -\frac{(1-\lambda)}{2h\alpha} |w|  + \bar{\lambda}B_0 B_1\]
so, by Gr{\"o}nwall lemma,
\begin{align*}
|w(t)| &\leq \text{exp} \left ( -\frac{(1-\lambda)}{2h\alpha} t \right ) |w(0)| + 2h \bar{\lambda}^2 \alpha B_0 B_1 \left (1 - \text{exp} \left ( -\frac{(1-\lambda)}{2h\alpha} t \right ) \right ) \\
&\leq \text{exp} \left ( -\frac{(1-\lambda)}{2h\alpha} t \right ) |w(0)| + h \beta_1
\end{align*}
where we define \(\beta_1 \coloneqq 2 \bar{\lambda}^2 \alpha B_0 B_1\). Hence
\begin{align*}
|\ddot{u}(t)| &= |\dot{v}(t)| \\
&= \frac{1}{h \alpha} |w(t)| \\
&\leq  \frac{1}{h \alpha} \text{exp} \left ( -\frac{(1-\lambda)}{2h\alpha} t \right ) |w(0)| + \frac{\beta_1}{\alpha} \\
&= \frac{|(1-\lambda)\sv_0 - f(\su_0)|}{h \alpha} \text{exp} \left ( -\frac{(1-\lambda)}{2h\alpha} t \right ) + \frac{\beta_1}{\alpha}
\end{align*}
thus setting \(C^{(2)}_1 = |(1-\lambda)\sv_0 - f(\su_0)|/\alpha\) and \(C^{(2)}_1 = \beta_1 / \alpha\) gives the desired result.
Further,
\begin{align*}
|\dot{u}(t)| &= |v(t)| \\
&\leq \bar{\lambda}(|w(t)| + |f(u(t))|) \\
&\leq \bar{\lambda}(|w(0)| + h \beta_1 + B_0)
\end{align*}
hence we deduce the existence of \(C^{(1)}\). Now define \(z \coloneqq \dot{w}\) then 
\[\dot{z} = - \frac{(1-\lambda)}{h \alpha} z - \bar{\lambda} Df(u)z + G(u,v,w)\]
where we define \(G(u,v,w) \coloneqq - \bar{\lambda}(Df(u)(Df(u)v) + D^2f(u)[v,w] + D^2f(u)[Df(u)v,f(u)])\). Using 
Assumption \ref{assump:one} and our bounds on \(w\) and \(v\), we deduce that there is a constant \(C > 0\) independent 
of \(h\) such that
\[|G(u,v,w)| \leq C\]
hence
\begin{align*}
\frac{1}{2} \frac{d}{dt} |z|^2 &= - \frac{(1-\lambda)}{h \alpha} |z|^2 - \bar{\lambda} \langle z, Df(u)z \rangle + \langle z, G(u,v,w) \rangle \\
&\leq - \frac{(1-\lambda)}{h \alpha} |z|^2 + \bar{\lambda}B_1 |z|^2 + C|z| \\
&\leq - \frac{(1-\lambda)}{2 h \alpha} |z|^2 + C|z|
\end{align*}
as before. Thus we find 
\[\frac{d}{dt} |z| \leq  - \frac{(1-\lambda)}{2 h \alpha} |z| + C\]
so, by Gr{\"o}nwall lemma,
\[|z(t)| \leq \text{exp} \left ( -\frac{(1-\lambda)}{2h\alpha} t \right ) |z(0)| + h \beta_2\]
where we define \(\beta_2 \coloneqq 2 \bar{\lambda} \alpha C \). Recall that 
\[\dddot{u} = \ddot{v} = - \frac{1}{h \alpha} \dot{w} = - \frac{1}{h \alpha} z\]
and note 
\[|z(0)| \leq \frac{(1-\lambda)|(1-\lambda)\sv_0 - f(\su_0)|}{h \alpha} + B_1 |\sv_0|\]
hence we find 
\[|\dddot{u}(t)| \leq \left ( \frac{(1-\lambda)|(1-\lambda)\sv_0 - f(\su_0)}{h^2 \alpha^2} + \frac{B_1 |\sv_0|}{h \alpha} \right ) \text{exp} \left ( -\frac{(1-\lambda)}{2h\alpha} t \right ) + \frac{\beta_2}{\alpha}.\]
Thus we deduce that there is a constant \(C^{(3)}_1 > 0\) independent of \(h\) such that 
\[|\dddot{u}(t)| \leq \frac{C^{(3)}_1}{h^2} \text{exp} \left ( -\frac{(1-\lambda)}{2h\alpha} t \right ) + C^{(3)}_2\]
as desired where \(C^{(3)}_2 = \beta_2 / \alpha\).
\end{proof}

\section*{Appendix C.}
\label{sec:AC}

For the results of Section \ref{sec:IM} we make the 
following assumption on the size of \(h\). Recall first that 
by Assumption \ref{assump:one} there are constants \(B_0,B_1,B_2 > 0\)
such that
\[\|D^{j-1}f\| = \|D^j \Phi\| \leq B_{j-1}\]
for \(j=1,2,3\).

\begin{lemma}
\label{lemma:gamma}
Suppose \(h>0\) is small enough such that
\[\lambda + hB_1(a + \lambda \bar{\lambda}) < 1\]
then there is a \(\tau_1 > 0\) such that for any \(\gamma \in [\tau_1,\infty)\)
\begin{equation}
\label{eq:gamma_bound}
(\lambda + hB_1(a + \lambda \bar{\lambda})) \gamma  + \bar{\lambda}B_0B_1(a + \bar{\lambda}) \leq \gamma.
\end{equation}
\end{lemma}
Using Lemma \ref{lemma:gamma} fix \(\gamma \in [\tau_1, \infty)\) and define the constants
\begin{align}
\label{eq:alpha_const}
\begin{split}
K_1 &\coloneqq \bar{\lambda}B_0 + h\gamma \\
K_3 &\coloneqq B_0 + \lambda K_1 \\
\alpha_2 &\coloneqq h^2 (\lambda + h a B_1), \\
\alpha_1 &\coloneqq \lambda - 1 + h \left ( B_1(\bar{\lambda} + a(1+h \bar{\lambda}B_1)) + \lambda \bar{\lambda} (B_1 + hB_2K_3) + ha(aB_2K_1 + B_1 \bar{\lambda}(B_1 + hB_2K_3) \right ), \\
\alpha_0 &\coloneqq aB_2K_1(1 + ha\bar{\lambda}B_1) + \bar{\lambda}(aB_1^2 + B_2K_3) + \bar{\lambda}^2B_1(1+haB_1)(B_1 + hB_2K_3).
\end{split}
\end{align}

\begin{lemma}
\label{lemma:delta}
Suppose \(h>0\) is small enough such that
\[\alpha_1^2 > 4 \alpha_2 \alpha_0, \quad \alpha_1 < 0\]
then there are \(\tau_2^{\pm} > 0\) such that for any \(\delta \in (\tau_2^-,\tau_2^+]\)
\begin{equation}
\label{eq:delta_bound}
\alpha_2 \delta^2 + \alpha_1 \delta + \alpha_0 \leq 0.
\end{equation}
\end{lemma}
Using Lemma \ref{lemma:delta} fix \(\delta \in (\tau_2^-,\tau_2^+]\). We make the following assumption on the size of the learning rate \(h\) which
is achievable since \(\lambda \in (0,1)\).

\begin{assumption}
\label{assump:h_small}
Let Assumption \ref{assump:one} hold and suppose \(h > 0\) is small enough such that the assumptions of Lemmas \ref{lemma:gamma}, \ref{lemma:delta} hold. 
Define \(K_2 \coloneqq \bar{\lambda}B_1 + h \delta\) and suppose \(h > 0\) is small enough such that
\begin{equation}
\label{eq:contraction1_const}
c \coloneqq h(\lambda K_2 + B_1(1+haK_2)) < 1.
\end{equation}
Define constants
\begin{align}
\label{eq:contraction2_const}
\begin{split}
Q_1 &\coloneqq \lambda \delta + a( B_1 K_2 + B_2 K_1 (1 + haK_2 )) + \bar{\lambda}( ( B_1 + hB_2K_3 )(\lambda K_2 + B_1 (1+ haK_2)) + B_2K_3  ), \\
Q_2 &\coloneqq h(a(B_1 + haB_2K_1) + \bar{\lambda} (\lambda + haB_1)(B_1+hB_2K_3)), \\
Q_3 &\coloneqq h(\lambda K_2 + B_1(1+haK_2)), \\
\mu &\coloneqq \lambda + Q_2 + \frac{h^2(\lambda + haB_1)Q_1}{1-Q_3}.
\end{split} 
\end{align}
Suppose \(h > 0\) is small enough such that
\begin{equation}
\label{eq:contraction}
Q_3 <  1, \quad \mu < 1.
\end{equation}
Lastly assume \(h > 0\) is small enough such that 
\begin{equation}
\label{eq:exp_attract}
\lambda + h^2 \lambda \delta < 1.
\end{equation}
\end{assumption}

\begin{proof}[of Lemma \ref{lemma:gamma}.]
Since \(\lambda + hB_1(a + \lambda \bar{\lambda}) < 1\) and \(\bar{\lambda}B_0B_1(a + \bar{\lambda}) > 0\) the line defined by
\[(\lambda + hB_1(a + \lambda \bar{\lambda})) \gamma  + \bar{\lambda}B_0B_1(a + \bar{\lambda})\] 
will intersect the identity line at a positive \(\gamma\) and lie below it thereafter.
Hence setting 
\[\tau_1 = \frac{\bar{\lambda}B_0B_1(a + \bar{\lambda})}{1 - \lambda + hB_1(a + \lambda \bar{\lambda})}\]
completes the proof.
\end{proof}

\begin{proof}[of Lemma \ref{lemma:delta}.]
Note that since \(\alpha_2 > 0\), the parabola defined by
\[\alpha_2 \delta^2 + \alpha_1 \delta + \alpha_0\]
is upward-pointing and has roots
\[\zeta_{\pm} = \frac{-\alpha_1 \pm \sqrt{\alpha_1^2 - 4\alpha_2\alpha_0}}{2\alpha_2}.\]
Since \(\alpha_1^2 > 4\alpha_2\alpha_0\), \(\zeta_{\pm} \in \R\) with \(\zeta_+ \neq \zeta_-\). Since \(\alpha_1 < 0\), \(\zeta_+ > 0\) hence 
setting \(\tau_2^+ = \zeta_+\) and \(\tau_2^- = \max \{0, \zeta_-\}\) completes the proof.
\end{proof}

The following proof refers to four lemmas whose statement and proof follow it.

\begin{proof}[of Theorem \ref{thm:invman_existence}.]
Define \(\tau > 0\) as the maximum \(h\) such that Assumption \ref{assump:h_small} holds. The contraction mapping principle together with Lemmas \ref{lemma:well_defined}, \ref{lemma:gamma_to_gamma}, and \ref{lemma:T_contraction} show that 
the operator \(T\) defined by \eqref{eq:p} and \eqref{eq:Tg} has a unique fixed point in \(\Gamma\).
Hence, from its definition and equation (\ref{eq:two_step}b), we immediately obtain the existence result. We now show exponential attractivity. Recall the definition of the operator \(T\) namely equations \eqref{eq:p}, \eqref{eq:Tg}:
\begin{align*}
p &= \xi + h z_g (\xi) \\
(Tg)(p) &= \lambda g(\xi) + a I^{(1)}_g(\xi) - \bar{\lambda} I^{(2)}_g(\xi).
\end{align*}
Let \(g \in \Gamma\) be the fixed point of \(T\) and set
\begin{align*}
p &= \su_n + h z_g(\su_n) \\
g(p) &= \lambda g(\su_n) + a I^{(1)}_g(\su_n) - \bar{\lambda}I^{(2)}_g(\su_n).
\end{align*}
Then 
\begin{align*}
|\sv_{n+1} - \bar{\lambda}f(\su_{n+1}) - hg(\su_{n+1})| &\leq |\sv_{n+1} - \bar{\lambda}f(\su_{n+1}) - hg(p)| + h|g(p) - g(\su_{n+1})| \\
&\leq |\sv_{n+1} - \bar{\lambda}f(\su_{n+1}) - hg(p)| + h \delta |p - \su_{n+1}|
\end{align*}
since \(g \in \Gamma\). Since, by definition,
\[\sv_{n+1} = \lambda \sv_n + f(\su_n + ha\sv_n)\]
we have,
\begin{align*}
|\sv_{n+1} - \bar{\lambda}f(\su_{n+1}) - hg(p)| &= |\lambda \sv_n + f(\su_n + ha\sv_n) - \bar{\lambda}f(\su_{n+1}) - h(\lambda g(\su_n) + a I^{(1)}_g(\su_n) - \bar{\lambda}I^{(2)}_g(\su_n))| \\
&= \lambda |\sv_n - \bar{\lambda}f(\su_n) - hg(\su_n)|
\end{align*}
by noting that
\begin{align*}
&f(\su_n + ha\sv_n) = f(\su_n) + ha I^{(1)}_g(\su_n) \\
&f(\su_{n+1}) = f(\su_n) + h I^{(2)}_g(\su_n).
\end{align*}
From definition,
\[\su_{n+1} = \su_n + h \lambda \sv_n + h f(\su_n + ha \sv_n)\]
thus
\begin{align*}
|p-\su_{n+1}| &= |\su_n + hz_g(\su_n) - \su_n - h\lambda \sv_n - hf(\su_n + ha\sv_n)| \\
&= h | \lambda (\bar{\lambda} f(\su_n) + hg(\su_n)) + f(\su_n + ha\sv_n) - \lambda \sv_n - f(\su_n + ha\sv_n)| \\
&= h \lambda |\sv_n - \bar{\lambda}f(\su_n) - hg(\su_n)|.
\end{align*}
Hence
\[|\sv_{n+1} - \bar{\lambda}f(\su_{n+1}) - hg(\su_{n+1})| \leq (\lambda + h^2 \lambda \delta) |\sv_n - \bar{\lambda}f(\su_n) - hg(\su_n)|\]
as desired. By Assumption \ref{assump:h_small}, \(\lambda + h^2 \lambda \delta < 1\).
\end{proof}

The following lemma gives basic bounds which are used in the proof of Lemmas \ref{lemma:well_defined}, \ref{lemma:gamma_to_gamma}, \ref{lemma:T_contraction}.

\begin{lemma}
\label{lemma:easy_bounds}
Let \(g,q \in \Gamma\) and \(\xi,\eta \in \R^d\) then the quantities defined by \eqref{eq:wg}, \eqref{eq:zg}, \eqref{eq:I1}, \eqref{eq:I2} satisfy the following:
\begin{align*}
|w_g(\xi)| &\leq  K_1, \\
|w_g(\xi) - w_g(\eta)| &\leq  K_2 |\xi-\eta|, \\
|w_g(\xi) - w_q(\xi)| &\leq h|g(\xi) - q(\xi)|, \\
|z_g(\xi)| &\leq K_3, \\
|z_g(\xi) - z_g(\eta)| &\leq \left ( \lambda K_2 + B_1 \left (1 + haK_2 \right ) \right )|\xi - \eta|, \\
|z_g(\xi) - z_q(\xi)| &\leq h \left ( \lambda + h a B_1 \right )|g(\xi) - q(\xi)|, \\
|I_g^{(1)}(\xi)| &\leq B_1 K_1, \\
|I_g^{(1)}(\xi) - I_g^{(1)}(\eta)| &\leq  ( B_1 K_2 + B_2 K_1 (1 + haK_2 ))|\xi - \eta|, \\
|I_g^{(1)}(\xi) - I_q^{(1)}(\xi)| &\leq h ( B_1 + haB_2 K_1 )|g(\xi) - q(\xi)|, \\
|I_g^{(2)}(\xi)| &\leq B_1 K_3 \\
|I_g^{(2)}(\xi) - I_g^{(2)}(\eta)| &\leq ( ( B_1 + hB_2K_3 )(\lambda K_2 + B_1 (1+ haK_2)) + B_2K_3  )|\xi-\eta|, \\
|I_g^{(2)}(\xi) - I_q^{(2)}(\xi)| &\leq h(\lambda + hB_1a)(B_1 + hB_2K_3)|g(\xi) - q(\xi)|.
\end{align*}
\end{lemma}
\begin{proof}
These bounds relay on applications of the triangle inequality together with boundedness of \(f\)
and its derivatives as well as the fact that functions in \(\Gamma\) are bounded and Lipschitz. 
To illustrate the idea, we will prove the bounds for \(w_g, w_q, I^{(1)}_g,\) and \(I^{(1)}_q\). 
To that end,
\begin{align*}
|w_g(\xi)| &= |\bar{\lambda} f(\xi) + hg(\xi)| \\
&\leq \bar{\lambda} |f(\xi)| + h |g(\xi)| \\
&\leq \bar{\lambda} B_0 + h \gamma \\
&= K_1
\end{align*}
establishing the first bound. For the second,
\begin{align*}
|w_g(\xi) - w_g(\eta)| &\leq \bar{\lambda}|f(\xi) - f(\eta)| + h |g(\xi) - g(\eta)| \\
&\leq \bar{\lambda} B_1 |\xi - \eta| + h \delta |\xi - \eta| \\
&= K_2 |\xi - \eta|
\end{align*}
as desired. Finally,
\begin{align*}
|w_g(\xi) - w_q(\xi)| &= |\bar{\lambda}f(\xi) + h g(\xi) - \bar{\lambda}f(\xi) - h q(\xi)| \\
&= h |g(\xi) - q(\xi)|
\end{align*}
as desired. We now turn to the bounds for \(I^{(1)}_g, I^{(1)}_q\),
\begin{align*}
|I^{(1)}_g(\xi)| &\leq \int_0^1 | Df(\xi + shaw_g(\xi))||w_g(\xi)|ds \\
&\leq \int_0^1 B_1 K_1 ds \\
&= B_1 K_1
\end{align*}
establishing the first bound. For the second bound,
\begin{align*}
|I^{(1)}_g(\xi) - I^{(1)}_g(\eta)| &\leq \int_0^1 |Df(\xi + shaw_g(\xi))w_g(\xi) - Df(\eta + shaw_g(\eta))w_g(\xi)|ds \\
&\;\;\;\;+ \int_0^1  |Df(\eta + shaw_g(\eta))w_g(\xi) - Df(\eta + shaw_g(\eta))w_g(\eta)|ds \\
&\leq K_1 B_2 \int_0^1 (|\xi - \eta| + sha|w_g(\xi) - w_g(\eta)|)ds + B_1 |w_g(\xi) - w_g(\eta)| \\
&\leq K_1 B_2 (|\xi - \eta| + ha K_2 |\xi - \eta| ) + B_1 K_2 |\xi - \eta| \\
&= (B_1 K_2 + B_2 K_1(1 + haK_2))|\xi - \eta|
\end{align*}
as desired. Finally
\begin{align*}
|I^{(1)}_g(\xi) - I^{(1)}_q(\xi)| &\leq \int_0^1 |Df(\xi + shaw_g(\xi))w_g(\xi) - Df(\xi + shaw_g(\xi))w_q(\xi)|ds \\
 &\;\;\;\;+ \int_0^1 |Df(\xi + shaw_g(\xi))w_q(\xi) - Df(\xi + shaw_q(\xi))w_q(\xi)|ds \\
 &\leq B_1 \int_0^1 |w_g(\xi) - w_q(\xi)|ds + K_1 B_2 \int_0^1 |\xi + shaw_g(\xi) - \xi - shaw_q(\xi)|ds \\
 &\leq h B_1 |g(\xi) - q(\xi)| + h^2 a B_2 K_1 |g(\xi) - q(\xi)| \\
 &= h(B_1 + ha B_2 K_1)|g(\xi) - q(\xi)|
\end{align*}
as desired. The bounds for \(z_g,z_q,I^{(2)}_g,\) and \(I^{(2)}_q\) follow similarly.
\end{proof}

We also need the following three lemmas:

\begin{lemma}
\label{lemma:well_defined}
Suppose Assumption \ref{assump:h_small} holds. For any \(g \in \Gamma\) and \(p \in \Rd\) there exists a unique
\(\xi \in \Rd\) satisfying \eqref{eq:p}.
\end{lemma}

\begin{lemma}
\label{lemma:gamma_to_gamma}
Suppose Assumption \ref{assump:h_small} holds. The operator \(T\) defined by \eqref{eq:Tg} satisfies \({T : \Gamma \to \Gamma}\).
\end{lemma}

\begin{lemma}
\label{lemma:T_contraction}
Suppose Assumption \ref{assump:h_small} holds. For any \(g_1,g_2 \in \Gamma\), we have
\[\|Tg_1 - Tg_2\|_\Gamma \leq \mu \|g_1 - g_2\|_\Gamma\]
where \(\mu < 1\).
\end{lemma}

Now we prove these three lemmas. 

\begin{proof}[of Lemma \ref{lemma:well_defined}.]
Consider the iteration of the form
\[\xi^{k+1} = p - h z_g(\xi^k).\]
For any two sequences \(\{\xi^k\}\), \(\{\eta^k\}\) generated by this iteration we have, by Lemma \ref{lemma:easy_bounds},
\begin{align*}
|\xi^{k+1} - \eta^{k+1}| &\leq h |z_g(\eta^k) - z_g(\xi^k)| \\
&\leq h(\lambda K_2 + B_1(1+haK_2))|\xi^k - \eta^k| \\
&= c |\xi^k - \eta^k|
\end{align*}
which is a contraction by \eqref{eq:contraction1_const}.
\end{proof}

\begin{proof}[of  Lemma \ref{lemma:gamma_to_gamma}.]
Let \(g \in \Gamma\) and \(p \in \R^d\) then by Lemma \ref{lemma:well_defined} there is a unique \(\xi \in \R^d\) such 
that \eqref{eq:p} is satisfied. Then 
\begin{align*}
|(Tg)(p)| &\leq \lambda |g(\xi)| + a |I^{(1)}_g(\xi)| + \tilde{\lambda} |I^{(2)}_g(\xi)| \\
&\leq \lambda \gamma  + a B_1(\tilde{\lambda}B_0 + h \gamma) + \tilde{\lambda}B_1(\lambda(\tilde{\lambda}B_0 + h\gamma) + B_0) \\
&= (\lambda + hB_1(a + \lambda \tilde{\lambda})) \gamma  + \tilde{\lambda}B_0B_1(a + \tilde{\lambda}) \\
&\leq \gamma
\end{align*}
with the last inequality following from \eqref{eq:gamma_bound}.

Let \(p_1, p_2 \in \R^d\) then, by Lemma \ref{lemma:well_defined}, there exist \(\xi_1,\xi_2 \in \R^d\)
such that \eqref{eq:p} is satisfied with \(p=\{p_1,p_2\}\). Hence, by Lemma \ref{lemma:easy_bounds},
\begin{align*}
|(Tg)(p_1) - (Tg)(p_2)| &\leq \lambda |g(\xi_1) - g(\xi_2)| + a |I^{(1)}_g(\xi_1) - I^{(1)}_g(\xi_2)| + \tilde{\lambda} |I^{(2)}_g(\xi_1) - I^{(2)}_g(\xi_2)| \\
&\leq K |\xi_1 - \xi_2|
\end{align*}
where we define 
\[
K \coloneqq \lambda \delta + a( B_1 K_2 + B_2 K_1 (1 + haK_2 )) + \tilde{\lambda} ( ( B_1 + hB_2K_3 )(\lambda K_2 + B_1 (1+ haK_2)) + B_2K_3  ).
\]
Now, using \eqref{eq:p} and the proof of Lemma \ref{lemma:well_defined},
\begin{align*}
|\xi_1 - \xi_2| &\leq |p_1 - p_2| + h |z_g(\xi_1) - z_g(\xi_2)| \\
&\leq |p_1 - p_2| + c |\xi_1 - \xi_2|.
\end{align*}
Since \(c < 1\) by \eqref{eq:contraction1_const}, we obtain
\[|\xi_1 - \xi_2| \leq \frac{1}{1-c}|p_1-p_2|\]
thus 
\[|(Tg)(p_1) - (Tg)(p_2)| \leq \frac{K}{1-c}|p_1-p_2| \leq \delta |p_1 - p_2|.\]
To see the last inequality, we note that
\[\frac{K}{1-c} \leq \delta \iff K - \delta (1-c) \leq 0\]
and \(K - \delta (1-c) = \alpha_2 \delta ^2 + \alpha_1 \delta + \alpha_0\) by \eqref{eq:alpha_const} hence \eqref{eq:delta_bound} gives 
the desired result.
\end{proof}

\begin{proof}[of Lemma \ref{lemma:T_contraction}.]
By Lemma \ref{lemma:well_defined}, for any \(p \in \R^d\) and \(g_1,g_2 \in \Gamma\), there are
\(\xi_1,\xi_2 \in \R^d\) such that
\begin{align*}
p &= \xi_j + h z_{g_j}(\xi_j)\\
(Tg_j)(p) &= \lambda g_j (\xi_j) + a I^{(1)}_{g_j}(\xi_j) - \tilde{\lambda} I^{(2)}_{g_j}(\xi_j)
\end{align*}
for \(j=1,2\). Then
\[|(Tg_1)(p) - (Tg_2)(p)| \leq \lambda |g_1(\xi_1) - g_2(\xi_2)| + a |I^{(1)}_{g_1}(\xi_1) - I^{(1)}_{g_2}(\xi_2)| + \tilde{\lambda}|I^{(2)}_{g_1}(\xi_1) - I^{(2)}_{g_2}(\xi_2)|.\]
Note that
\begin{align*}
|g_1(\xi_1) - g_2(\xi_2)| &= |g_1(\xi_1) - g_2(\xi_2) - g_2(\xi_1) + g_2(\xi_1)| \\
&\leq |g_1(\xi_1) - g_2(\xi_1)| + \delta|\xi_1 - \xi_2|.
\end{align*}
Similarly, by Lemma \ref{lemma:easy_bounds},
\begin{align*}
|I^{(1)}_{g_1}(\xi_1) - I^{(1)}_{g_2}(\xi_2)| &= |I^{(1)}_{g_1}(\xi_1) - I^{(1)}_{g_2}(\xi_2) - I^{(1)}_{g_2}(\xi_1) + I^{(1)}_{g_2}(\xi_1)| \\
&\leq |I^{(1)}_{g_1}(\xi_1) - I^{(1)}_{g_2}(\xi_1)| + |I^{(1)}_{g_2}(\xi_1) - I^{(1)}_{g_2}(\xi_2)| \\
&\leq h ( B_1 + haB_2 K_1)|g_1(\xi_1) - g_2(\xi_1)| + ( B_1 K_2 + B_2 K_1 (1 + haK_2 ))|\xi_1 - \xi_2|
\end{align*}
Finally,
\begin{align*}
|I^{(2)}_{g_1}(\xi_1) - I^{(2)}_{g_2}(\xi_2)| &= |I^{(2)}_{g_1}(\xi_1) - I^{(2)}_{g_2}(\xi_2) - I^{(2)}_{g_2}(\xi_1) + I^{(2)}_{g_2}(\xi_1)| \\
&\leq |I^{(2)}_{g_1}(\xi_1) - I^{(2)}_{g_2}(\xi_1)| + |I^{(2)}_{g_2}(\xi_1) - I^{(2)}_{g_2}(\xi_2)| \\
&\leq h(\lambda + hB_1a)(B_1 + hB_2K_3)|g_1(\xi_1) - g_2(\xi_1)| + \\
&+( ( B_1 + hB_2K_3 )(\lambda K_2 + B_1 (1+ haK_2)) + B_2K_3  )|\xi_1 - \xi_2|
\end{align*}
Putting these together and using \eqref{eq:contraction2_const}, we obtain
\[|(Tg_1)(p) - (Tg_2)(p)| \leq ( \lambda + Q_2 )|g_1(\xi_1) - g_2(\xi_1)| + Q_1 |\xi_1 - \xi_2|. \]
Now, by Lemma \ref{lemma:easy_bounds},
\begin{align*}
|\xi_1 - \xi_2| &\leq h |z_{g_1}(\xi_1) - z_{g_2}(\xi_2) - z_{g_2}(\xi_1) + z_{g_2}(\xi_1)| \\
&\leq h (|z_{g_1}(\xi_1) - z_{g_2}(\xi_1)| + |z_{g_2}(\xi_1) - z_{g_2}(\xi_2)|) \\
&\leq h^2(\lambda + haB_1)|g_1(\xi) - g_2(\xi_1)| + h(\lambda K_2 + B_1(1+haK_2))|\xi_1 - \xi_2| \\
&= h^2(\lambda + haB_1)|g_1(\xi) - g_2(\xi_1)| + Q_3|\xi_1 - \xi_2|
\end{align*}
using \eqref{eq:contraction2_const}. Since, by \eqref{eq:contraction}, \(Q_3 < 1\), we obtain
\[|\xi_1 - \xi_2| \leq \frac{h^2(\lambda + haB_1)}{1-Q_3} |g_1(\xi_1) - g_2(\xi_1)|\]
and thus
\begin{align*}
|(Tg_1)(p) - (Tg_2)(p)| &\leq \left ( \lambda + Q_2 + \frac{h^2(\lambda + haB_1)Q_1}{1-Q_3} \right ) |g_1(\xi_1) - g_2(\xi_1)| \\
&= \mu |g_1(\xi_1) - g_2(\xi_1)|
\end{align*}
by \eqref{eq:contraction2_const}. Taking the supremum over \(\xi_1\) then over \(p\) gives the desired result.
Since \(\mu < 1\) by \eqref{eq:contraction}, we obtain that \(T\) is a contraction on \(\Gamma\).
\end{proof}

\section*{Appendix D}
\label{sec:AD}

We consider the equation
\begin{align*}
&\ddot{u} + 2 \sqrt{\mu} \dot{u} + \nabla \Phi(u) = 0 \\
&u(0) = \su_0, \quad \dot{u}(0) = \sv_0.
\end{align*}
Set \(v = \dot{u}\) then we have
\[
\begin{bmatrix}
\dot{u} \\
\dot{v}
\end{bmatrix}
=
\begin{bmatrix}
v \\
-2\sqrt{\mu} v - \nabla \Phi(u)
\end{bmatrix}.
\]
Define the maps
\[f_1(u,v) \coloneqq \begin{bmatrix}
v \\
-2 \sqrt{\mu} v
\end{bmatrix},
\quad
f_2(u,v) \coloneqq \begin{bmatrix}
0 \\
- \nabla \Phi(u)
\end{bmatrix}\]
then
\[
\begin{bmatrix}
\dot{u} \\
\dot{v}
\end{bmatrix}
= f_1(u,v) + f_2(u,v).
\]
We first solve the system
\[
\begin{bmatrix}
\dot{u} \\
\dot{v}
\end{bmatrix}
= f_1(u,v).
\]
Clearly 
\[v(t) = e^{-2\sqrt{\mu}t} \sv_0\]
hence
\begin{align*}
u(t) &= \su_0 + \int_0^t e^{-2\sqrt{\mu}s} \sv_0 \: ds \\
&= \su_0 + \frac{1}{2 \sqrt{\mu}}  \left( 1 - e^{-2\sqrt{\mu}t} \right) \sv_0.
\end{align*}
This gives us the flow map
\[
\psi_1(\su, \sv; t) = \begin{bmatrix}
\su + \frac{1}{2 \sqrt{\mu}}  \left( 1 - e^{-2\sqrt{\mu}t} \right) \sv \\
e^{-2\sqrt{\mu}t} \sv
\end{bmatrix}.
\]
We now solve the system
\[
\begin{bmatrix}
\dot{u} \\
\dot{v}
\end{bmatrix}
= f_2(u,v).
\]
Clearly
\[u(t) = \su_0\]
hence
\[v(t) = \sv_0 - t \nabla \Phi(\su_0).\]
This gives us the flow map
\[\psi_2(\su,\sv;t) = \begin{bmatrix}
\su \\
\sv - t \nabla \Phi(\su)
\end{bmatrix}.
\]
The composition of the flow maps is then
\[(\psi_2 \circ \psi_1)(\su, \sv; t) = \begin{bmatrix}
\su + \frac{1}{2 \sqrt{\mu}}  \left( 1 - e^{-2\sqrt{\mu}t} \right) \sv \\
e^{-2\sqrt{\mu}t} \sv - t \nabla \Phi \left( \su + \frac{1}{2 \sqrt{\mu}}  \left( 1 - e^{-2\sqrt{\mu}t} \right) \sv \right)
\end{bmatrix}.\]
Mapping \(t\) to the time-step \(\sqrt{h}\) gives the numerical method \eqref{eq:wilson_disc}.

\vskip 0.2in
\bibliography{references}

\end{document}